\documentclass{article}

\usepackage{microtype}
\usepackage{graphicx}
\usepackage{booktabs} 

\usepackage{hyperref}

\usepackage[accepted]{icml2021}

\icmltitlerunning{Instance-Optimal Compressed Sensing via Posterior
Sampling}

\usepackage{subcaption}
\usepackage{amsmath}
\usepackage{amsthm}
\usepackage{amssymb}
\usepackage{bm}
\usepackage{tikz} 
\usepackage{caption}
\usepackage{subcaption}
\usepackage{enumitem}
\usetikzlibrary{positioning}
\usetikzlibrary{shapes}
\newtheorem{theorem}{Theorem}[section]
\newtheorem{lemma}[theorem]{Lemma}
\newtheorem{definition}[theorem]{Definition}

\newtheorem{claim}[theorem]{Claim}

\newcounter{example}[section]

\def\ci{\perp\!\!\!\perp}

\newcommand{\norm}[1]{\|#1\|}
\newcommand{\inner}[1]{\langle#1\rangle}
 
\newcommand{\wh}{\widehat}

\newcommand{\eps}{\varepsilon}

\newcommand{\R}{\mathbb{R}}

\newcommand{\RN}[1]{%
  \textup{\uppercase\expandafter{\romannumeral#1}}%
}

\newcommand{\vertiii}[1]{{\left\vert\kern-0.25ex\left\vert\kern-0.25ex\left\vert #1 
		\right\vert\kern-0.25ex\right\vert\kern-0.25ex\right\vert}}

\DeclareMathOperator*{\E}{\mathbb{E}}

\DeclareMathOperator{\supp}{supp}
\DeclareMathOperator{\esssup}{ess\,sup}

\DeclareMathOperator*{\argmin}{arg\,min}

\newcommand{\bx}{\mathbf{x}}

\newcommand{\xtilde}{\tilde{x}}

\newcommand{\ztilde}{\tilde{z}}

\newcommand{\xhat}{\wh{x}}
\newcommand{\zhat}{\wh{z}}

\newcommand{\cB}{\mathcal B}

\newcommand{\cN}{\mathcal N}

\newcommand{\cW}{\mathcal W}

\newcommand{\bbN}{\mathbb N}

\DeclareMathOperator{\cov}{Cov}
\DeclareMathOperator{\logcov}{\log \cov}

\usepackage{lineno}

\usepackage{enumitem}

\usepackage{bbm}
\usepackage{comment}
\usepackage{cleveref}
\allowdisplaybreaks

\graphicspath{ {./figures/}}
\usepackage{etoolbox}

\makeatletter

\newcommand{\define}[4][ignore]{%
  \ifstrequal{#1}{ignore}{}{
  \@namedef{thmtitle@#2}{#1}}%
  \@namedef{thm@#2}{#4}%
  \@namedef{thmtypen@#2}{lemma}%
  \newtheorem{thmtype@#2}[theorem]{#3}%
  \newtheorem*{thmtypealt@#2}{#3~\ref{#2}}%
}

\newcommand{\state}[1]{%
  \@namedef{curthm}{#1}
  \@ifundefined{thmtitle@#1}{
  \begin{thmtype@#1}
    }{
  \begin{thmtype@#1}[\@nameuse{thmtitle@#1}]
  }
    \label{#1}
    \@nameuse{thm@#1}
  \end{thmtype@#1}
  \@ifundefined{thmdone@#1}{
  \@namedef{thmdone@#1}{stated}%
  }{}
}

\newcommand{\restate}[1]{%
  \@namedef{curthm}{#1}
  \@ifundefined{thmtitle@#1}{
    \begin{thmtypealt@#1}
    }{
  \begin{thmtypealt@#1}[\@nameuse{thmtitle@#1}]
  }
    \@nameuse{thm@#1}
  \end{thmtypealt@#1}
  \@ifundefined{thmdone@#1}{
  \@namedef{thmdone@#1}{stated}%
  }{}
}

\newcommand{\thmlabel}[1]{
  \@ifundefined{thmdone@\@nameuse{curthm}}{\label{#1}
    }{\tag*{\eqref{#1}}}
}
\makeatother

\begin{document}

\twocolumn[
\icmltitle{Instance-Optimal Compressed Sensing via Posterior Sampling}




\begin{icmlauthorlist}
\icmlauthor{Ajil Jalal}{ece}
\icmlauthor{Sushrut Karmalkar}{cs}
\icmlauthor{Alexandros G. Dimakis}{ece}
\icmlauthor{Eric Price}{cs}
\end{icmlauthorlist}

\icmlaffiliation{ece}{University of Texas at Austin, Department of Electrical and Computer Engineering}
\icmlaffiliation{cs}{University of Texas at Austin, Department of Computer Science}

\icmlcorrespondingauthor{Ajil Jalal}{ajiljalal@utexas.edu}
\icmlcorrespondingauthor{Sushrut Karmalkar}{sushrutk@cs.utexas.edu}
\icmlcorrespondingauthor{Alexandros G. Dimakis}{dimakis@austin.utexas.edu}
\icmlcorrespondingauthor{Eric Price}{ecprice@cs.utexas.edu}

\icmlkeywords{Compressed Sensing, Bayesian Compressed Sensing,
Generative Priors, Deep Generative Models, Langevin Dynamics,
Posterior Sampling, Wassertein Distance}

\vskip 0.3in
]



\printAffiliationsAndNotice{Code and models available at:
\url{https://github.com/ajiljalal/code-cs-fairness}.}  
\begin{abstract}
  We characterize the measurement complexity of compressed sensing of
  signals drawn from a known prior distribution, even when the support
  of the prior is the entire space (rather than, say, sparse vectors).
  We show for Gaussian measurements and \emph{any} prior distribution
  on the signal, that the posterior sampling estimator achieves
  near-optimal recovery guarantees.  Moreover, this result is robust
  to model mismatch, as long as the distribution estimate (e.g., from
  an invertible generative model) is close to the true distribution in
  Wasserstein distance. We implement the posterior sampling
  estimator for deep generative priors using Langevin dynamics, and
  empirically find that it produces accurate estimates with more
  diversity than MAP.
\end{abstract}

\section{Introduction}

The goal of compressed sensing is to recover a structured signal from
a relatively small number of linear measurements. 
The setting of such linear inverse problems has numerous and diverse 
applications ranging from Magnetic Resonance Imaging~\cite{lustig2008compressed,lustig2007sparse}, 
neuronal spike trains~\cite{hegde2009compressive}
and efficient sensing cameras~\cite{duarte2008single}. 
Estimating a signal in $\R^n$ would in general require
$n$ linear measurements, but because real-world signals are
structured---i.e., compressible---one is often able to estimate them
with $m \ll n$ measurements.

Formally, we would like to estimate a ``signal'' $x^* \in \R^n$ from
noisy linear measurements,
\[
  y = Ax^* + \xi
\]
for a measurement matrix $A \in \R^{m \times n}$ and noise vector
$\xi \in \R^m$.  We will focus on the i.i.d.\ Gaussian setting, where
$A_{ij} \sim \cN(0, \frac{1}{m})$ and
$\xi_i \sim \cN(0, \frac{\sigma^2}{m})$, and one would like to recover
$\wh{x}$ from $(A, y)$ such that
\begin{align}\label{eq:l2l2}
  \norm{x^* - \wh{x}} \leq C\sigma
\end{align}
with high probability for some constant $C$.  When $x^*$ is $k$-sparse, this was shown
by~Cand\'es, Romberg, and Tao~\cite{candes2006stable} to be possible for $m$ at least
$O(k \log \frac{n}{k})$.

Over the past 15 years, compressed sensing has been extended in a wide
variety of remarkable ways, including by generalizing from sparsity to
other signal structures, such as those given by trees~\cite{chen12},
graphs~\cite{xu2011compressive}, manifolds~\cite{chen2010compressive, xu2008compressed}, or deep generative
models~\cite{bora2017compressed,asim2019invertible}.  These are all essentially frequentist
approaches to the problem: they define a small \emph{set} of
``structured'' signals $x$, and ask for recovery of every such signal.

Such set-based approaches have limitations.  For
example,~\cite{bora2017compressed} uses the structure given by a deep
generative model $G: \R^k \to \R^n$; with $O(k d \log n)$ measurements
for $d$-layer networks, accurate recovery is guaranteed for every
signal $x^*$ near the range of $G$.  But this completely ignores the
\emph{distribution} over the range.  Generative models like
Glow~\cite{kingma2018glow} and pixelRNN~\cite{oord2016pixel} have seed length $k=n$ and range
equal to the entire $\R^n$.  Yet because these models are designed to approximate reality, and real images can be compressed, we know that compressed sensing is possible in principle.

This leads to the question: Given signals drawn from some
\emph{distribution} $R$, can we characterize the number of linear
measurements necessary for recovery, with both upper and lower bounds?
Such a Bayesian approach has previously been considered for
sparsity-inducing product
distributions~\cite{aeron2010information,zhou2014bayesian} but not
general distributions.

Second, suppose that we don't know the real distribution $R$, but
instead have an approximation $P$ of $R$ (e.g., from a GAN or
invertible generative model).  In what sense should $P$ approximate
$R$ for compressed sensing with good guarantees to be possible?

\begin{figure*}[t]
\begin{center}
\includegraphics[width=\textwidth]{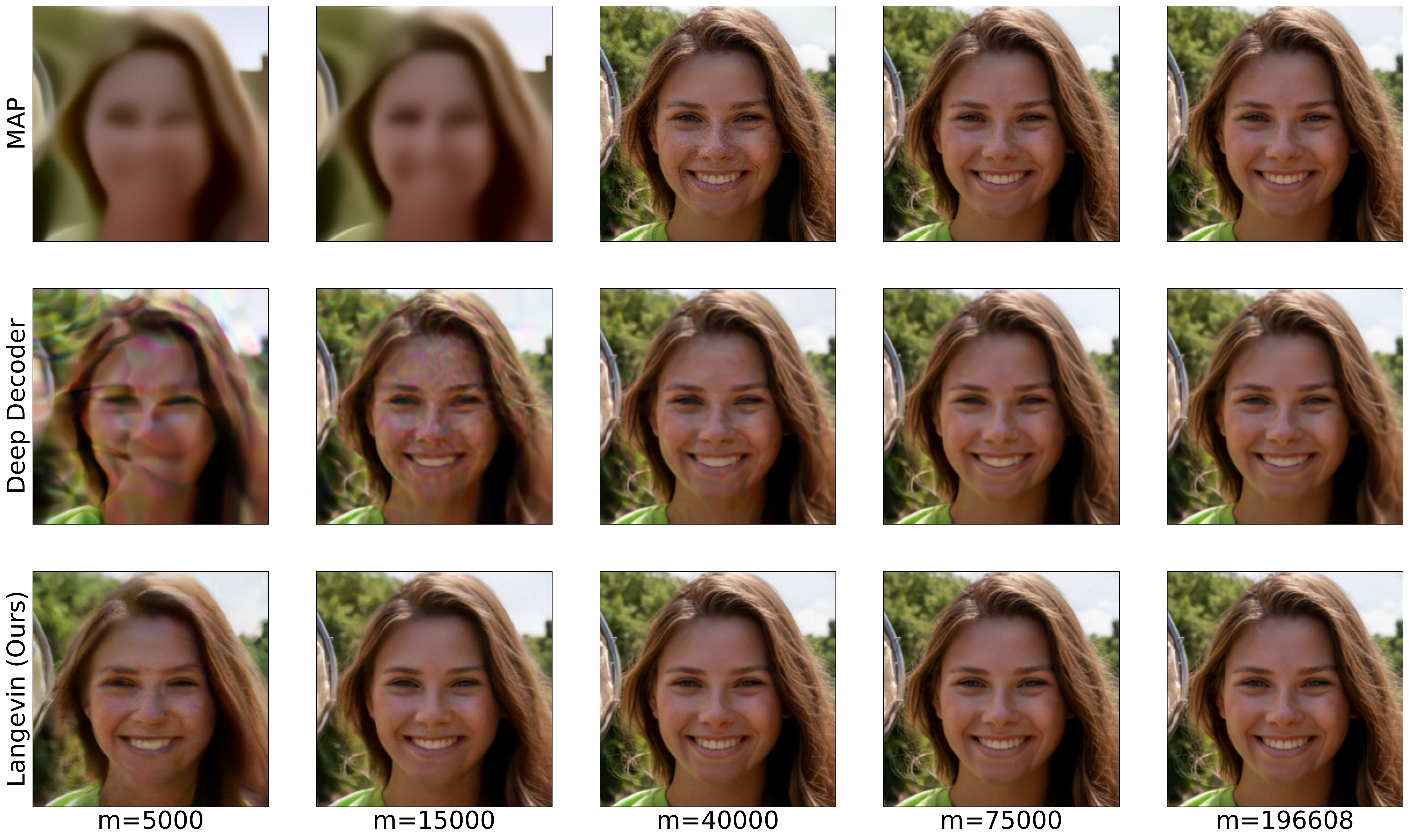}  
\end{center}
\caption{\small Reconstruction results on FFHQ for Gaussian measurements
		(here $n=256\times 256\times 3 = 196,608$ pixels), using an NCSNv2
		model. Each column shows the reconstruction obtained as the number
		of measurements $m$ varies.  The top row shows reconstructions by
		MAP, the middle row shows reconstruction by Deep-Decoder, and the
		bottom row shows reconstructions by Langevin dynamics, which is
		the practical implementation of our proposed posterior
		sampling estimator.}
\label{fig:cs-reconstr-ffhq}
\end{figure*}

\subsection{Contributions.}

Our main theorem is that posterior sampling is a near optimal
recovery algorithm for \emph{any} distribution.  Moreover, it is
sufficient to learn the distribution in Wasserstein distance.

\begin{theorem}\label{thm:intro}
  Let $R$ be an arbitrary distribution over an $\ell_2$ ball of radius
  $r$. Suppose that there exists an algorithm that uses an \emph{arbitrary
  measurement matrix} $A\in \R^{m \times n}$ with noise level $\sigma$
  and finds a reconstruction $\wh{x}$ such that
  \[
    \norm{x^* - \wh{x}} \lesssim \sigma \text{ with probability } \geq
    1 - \delta.
  \]
  Then posterior sampling (see Definition~\ref{defn: algorithm}) with respect to $R$ using
  $m' \geq O\left( m \log \left( 1 + \frac{m r^2 \norm{A}_\infty^2}{\sigma^2}
    \right) + \log \frac{1}{\delta}\right)$ Gaussian measurements of noise level $\sigma$ will
  output $\wh{x}$ satisfying
    \[
      \norm{x^* - \wh{x}} \lesssim \sigma \text{ with
      probability } \geq 1 - O(\delta).
    \]

    Moreover, the same holds for posterior sampling with respect
    to any distribution $P$ satisfying
    $\cW_p(R,P) \lesssim \sigma \delta^{1/p}$ for some $p \geq 1$.
\end{theorem}

This theorem comprises three main contributions: the introduction of
posterior sampling as \emph{a new algorithm} for recovery with a
generative prior; an \emph{upper bound} on the sample complexity of
the algorithm in terms of an approximate covering number that we
introduce; and an \emph{instance-optimal lower bound} in terms of the
same approximate covering number that (unlike previous lower bounds in
compressed sensing) applies to \emph{any} distribution of input
signals.

\begin{figure*}[t]
\begin{center}
  \includegraphics[width=0.9\textwidth]{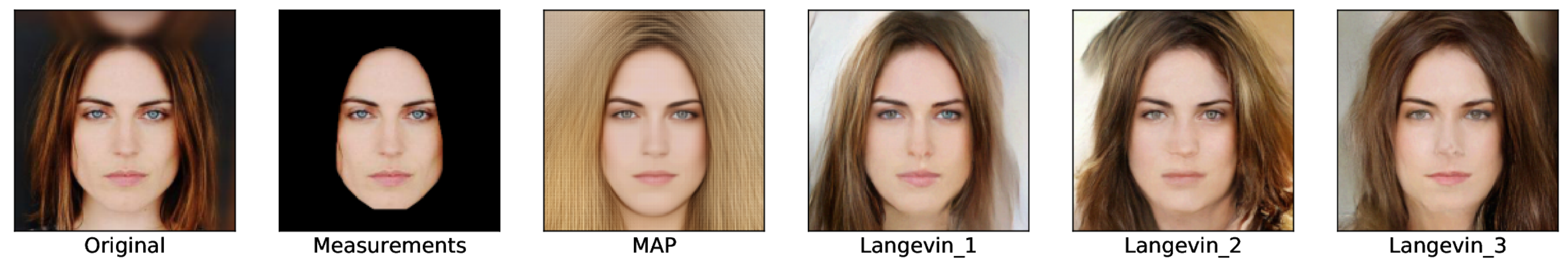}
\end{center}
\caption{\small Reconstruction results for inpainting on CelebA-HQ using
		Glow. The first column shows the original image, second column
		shows the measurements by removing the hair and background, the
		third column shows reconstruction by MAP, and the last three
		columns show samples from posterior sampling via Langevin
		dynamics. MAP produces the
		same washed out image all the time, whereas posterior sampling
		produces images with diversity.}
\label{fig:inpaint-hair}
\end{figure*}

\paragraph{Contribution 1: Approximate covering numbers.}
The covering number of a set is the smallest number of balls that can
cover the entire set.  Standard compressed sensing is closely tied to
the covering number $N_\eta(S)$ of the set $S$ of possible signals
$x$; for example, the set of unit-norm $k$-sparse vectors has $\log
N_\eta = \Theta(k \log \frac{n}{k})$, which is precisely why Cand\'es,
Romberg, and Tao use this many linear measurements to
achieve~\eqref{eq:l2l2}.

 For distributions, we need a different concept of covering number.
 As a motivating example, consider a distribution $R$ induced by a
 trivial \emph{linear} generative model, $ x = \Sigma z $ where $z
 \sim \cN(0, I_n)$ and $\Sigma$ is a fixed $n \times n$ matrix.
 Further suppose the singular values $\sigma_i$ of $\Sigma$ are
 Zipfian, so $\sigma_i = 1/i$.  In this case, $R$'s support is $\R^n$,
 so covering the entire support of $R$ is infeasible.  Instead we
 could denote by $\cov_{\eta, 0.01}(R)$ the minimum number of
 $\eta-$radius balls needed to cover 99\% of $R$. An elementary
 calculation shows
\[
  \log \cov_{\eta, 0.01}(R) = \Theta(1/\eta^2),
\]
which is (up to constants) precisely the number of linear measurements
you need to estimate $x$ to within $\eta$.  

We show that an \emph{approximate covering number} characterizes the
measurement complexity of compressed sensing a general distribution
$R$, and that recovery by \emph{posterior sampling} achieves this
bound.

\begin{definition}
  Let $R$ be a distribution on $\R^n$. For some parameters $\eta > 0,
  \delta\in \left[ 0,1 \right],$ we define the
  \emph{$(\eta,\delta)$-approximate covering number} of $R$ as
  \[
    \cov_{\eta,\delta}(R) := \min\left\{ k :  R\left[ \cup_{i=1}^k \cB(x_i,\eta)  \right]\geq 1 - \delta, x_i \in \R^n  \right\},
  \]
  where $\cB(x, \eta)$ is the $\ell_2$ ball of radius $\eta$ centered at $x$.
\end{definition}\label{defn: cov}

When $\delta = 0$, this is $N_\eta(\supp R)$, the standard covering
number of the support of $R$.  Having $\delta > 0$ allows meaningful
results for full-support distributions that are concentrated on
smaller sets. This also generalizes our previous results
in~\cite{bora2017compressed}, which depend on the covering numbers of low-dimensional
generative models.

\paragraph{Contribution 2: Recovery algorithm.}
The recovery algorithm we consider is posterior sampling:
\begin{definition}\label{defn: algorithm}
  Given an observation $y$, the \emph{posterior
  sampling} recovery algorithm with respect to $P$ outputs $\wh{x}$
  according to the posterior distribution $P(\cdot \mid y)$.
\end{definition}

\paragraph{Contribution 3: Sample complexity upper bound.}
Our main positive result is that posterior sampling achieves the
guarantees of equation~\eqref{eq:l2l2} for \emph{general}
distributions $R$, with $O(\log \cov_{\sigma,\delta}(R))$ measurements.
Not only this, but the algorithm is robust to model mismatch:
posterior sampling with respect to $P \neq R$ still works, as long
as $P$ and $R$ are close in Wasserstein distance:
\begin{theorem}[Upper bound]\label{thm:upperinformal}
  Let $P$, $R$ be distributions with $\cW_1(P, R) \leq \sigma$.  Let
  $x^* \sim R$, let $y$ be Gaussian measurements with noise level
  $\sigma$,  and let $\wh{x} \sim P(\cdot | y)$.  For any $\eta \geq
  \sigma$, with
  \[
    m \geq O(\log \cov_{\eta, 0.01}(R))
  \]
  measurements, the guarantee $ \norm{\wh{x} - x^*} \leq C \eta $
  is satisfied for some universal constant $C$ with $97\%$ probability
  over the signal $x$, measurement matrix $A$, noise $\xi$, and
  recovery algorithm $\wh{x}$.
\end{theorem}

\paragraph{Contribution 4: Sample complexity lower bound.}
Our second main result lower bounds the sample complexity for  
\emph{any} distribjution. This is, to our knowledge,
the first lower bound for compressed sensing that applies to arbitrary
distributions $R$.  Most lower bounds in the area are minimax, and
only apply to specific ``hard'' distributions
$R$~\cite{price20111+,candes2013well,iwen2010adaptive}; the closest
result we are aware of is~\cite{aeron2010information}, which
characterizes product distributions.

\begin{theorem}[Lower bound]\label{thm:lowerinformal}
  Let $R$ be any distribution over an $\ell_2$ ball of radius $r$, and
  consider any method to achieve $\norm{\wh{x} - x^*} \leq \eta$ with
  $99\%$ probability, using an \emph{arbitrary measurement matrix} $A
  \in \R^{m \times n}$ with noise level $\sigma$. This must have
  \[
    m \geq \frac{C'}{\log(1 + \frac{m r^2 \norm{A}_\infty^2}{\sigma^2} )} \log \cov_{C' \eta, 0.04} (R).
  \]
  for some constant $C' > 0$.
\end{theorem}

Note that Theorem~\ref{thm:upperinformal} and~\ref{thm:lowerinformal}
directly give Theorem~\ref{thm:intro}.
For more precisely stated and general versions of these results,
including dependence on the failure probability $\delta$, see
Theorems~\ref{thm: main} and~\ref{thm: lower}.

\subsection{Related Work}

Generative priors have shown great promise in compressed sensing and
other inverse problems, starting with~\cite{bora2017compressed}, who
generalized the theoretical framework of compressive sensing and
restricted eigenvalue
conditions~\cite{tibshirani1996regression,donoho2006compressed,bickel2009simultaneous,candes2008restricted,
hegde2008random,baraniuk2009random,baraniuk2010model,eldar2009robust}
for signals lying on the range of a deep generative
model~\cite{goodfellow2014generative, kingma2013auto}. 

Lower bounds
in~\cite{kamath2019lower,liu2019information,jalali2019solving}
established that the sample complexities in~\cite{bora2017compressed}
are order optimal.  The approach in~\cite{bora2017compressed} has been
generalized to tackle different inverse problems such as robust
compressed sensing~\cite{jalal2020robust}, phase
retrieval~\cite{hand2018phase,aubin2019exact,jagatap2019phase}, blind
image deconvolution~\cite{asim2018blind}, seismic
inversion~\cite{mosser2020stochastic}, one-bit
recovery~\cite{qiu2019robust,liu2020sample}, and blind
demodulation~\cite{hand2019global}.  Alternate algorithms for
reconstruction include sparse deviations from generative
models~\cite{dhar2018modeling}, task-aware compressed sensing~\cite{
kabkab2018task}, PnP~\cite{pandit2019inference, fletcher2018inference,
fletcher2018plug}, iterative projections~\cite{mardani2018deep},
OneNet~\cite{rick2017one} and Deep Decoder~\cite{heckel2018deep,
heckel2020compressive}.  The complexity of optimization algorithms
using generative models have been analyzed for
ADMM~\cite{gomez2019fast}, PGD ~\cite{hegde2018algorithmic},
layer-wise inversion~\cite{lei2019inverting}, and gradient
descent~\cite{hand2017global}.  Experimental results
in~\cite{asim2019invertible, whang2020compressed,
lindgren2020conditional} show that invertible models have superior
performance in comparison to low dimensional models.
See~\cite{ongie2020deep} for a more detailed survey on deep learning
techniques for compressed sensing.  A related line of work has
explored learning-based approaches to tackle classical problems in
algorithms and signal
processing~\cite{aamand2019learned,indyk2019learning,
metzler2017learned, hsu2018learning}.

Lower bounds for $\ell_2/\ell_2$ recovery of sparse vectors can be
found in~\cite{scarlett2016limits,price20111+, aeron2010information,
iwen2010adaptive, candes2013well}, and these are related to the lower
bound in~\eqref{thm:lowerinformal}.  The closest result is that
of~\cite{aeron2010information}, which characterizes the probability of
error and $\ell_2$ error of the reconstruction via covering numbers of
the probability distribution. Their approach uses the rate distortion
function of a scalar random variable $\bx$, and provides guarantees
for the product measure generated via an i.i.d. sequence of $\bx$.  A
Shannon theory for compressed sensing was pioneered
by~\cite{wu2012optimal,wu2011shannon}.  The $\delta-$Minkowski
dimension of a probability measure used
in~\cite{wu2012optimal,wu2011shannon,pesin2008dimension} can be
derived from our $(\varepsilon,\delta)-$covering number by taking the
limit $\varepsilon\to 0$.  \cite{reeves2012sampling} contains a
related theory of rate distortion for compressed sensing.  There is
also related work in the statistical physics community under different
assumptions on the signal
structure~\cite{zdeborova2016statistical,barbier2019optimal}.

\section{Background and Notation}

In this section, we introduce a few concepts that we will use
throughout the paper. $\norm{\cdot}$ refers to the $\ell_2$ norm
unless specified otherwise.  The metric we use to quantify the
similarity between distributions is the Wassertein distance. For two
probability distributions $\mu, \nu$ supported on $\Omega$, and for
any $p\geq 1$, the Wasserstein-$p$~\cite{villani2008optimal,
arjovsky2017wasserstein} and
Wasserstein-$\infty$~\cite{champion2008wasserstein} distances are
defined as:

{\small
\begin{align*}
  \cW_p ( \mu, \nu) &:= \inf_{\gamma \in \Pi(\mu,\nu)} \left( \E_{(u,v)
  \sim \gamma} \left[ \norm{ u - v}^p \right] \right)^{1/p}, \\
  \cW_{\infty}(\mu,\nu) &:= \inf_{\gamma \in \Pi(\mu,\nu)} \left(
    \underset{{(u,v) \in \Omega^2}}{\gamma\text{-}\esssup} \norm{u - v}
\right),
\end{align*}
} 
where $\Pi(\mu,\nu)$ denotes the set of joint distributions whose
marginals are $\mu,\nu$.
The above definition says that if $\cW_{\infty}(\mu,\nu)\leq \eps$,
and $(u,v) \sim \gamma$, then $\norm{u-v} \leq \eps$ almost surely.

We say that $y$ is generated from $x^*$ by a Gaussian measurement
process with $m$ measurements and noise level $\sigma$, if $y = Ax^* +
\xi$ where $\xi \sim \mathcal{N}(0, \frac{\sigma^2}{m} I_m)$ and
$A \in \mathbb{R}^{m \times n}$ with $A_{ij} \sim \mathcal{N}(0,
1/m)$.

\section{Upper Bound}

\subsection{Two-Ball Case} \label{exmp: two ball example}

For simplicity, we will first demonstrate our proof techniques in the
simple setting where $R=P$, the measurements are noiseless, and the
ground truth distribution $P$ is supported on two disjoint balls
(illustrated in Figure~\ref{fig:two ball example}). In this example,
two $\eta$ radius balls can cover the whole space, so the parameters
in Theorem~\ref{thm:upperinformal} will be $\sigma=0$ and
$\cov_{\eta,0}(P)=2$. Applying Theorem~\ref{thm:upperinformal} on $P$
tells us that a constant number of measurements is sufficient for
posterior sampling to get $O(\eta)$-close to the ground truth, i.e.,
to return an element of the correct ball. We will now prove this
claim.

  Let $B_{0}, B_{\xtilde}$ denote $\eta$-radius balls centered at $0, \xtilde \in \R^n$ respectively.
  Suppose $P = 0.5 P_{0} + 0.5 P_1$, where $P_{0}, P_1$, are uniform distributions on $B_{0}, B_{\xtilde}$.
  The centers of the balls are separated by a distance $d \gg \eta$.

  The ground truth $x^*$ will be sampled from $P$.  For a fixed matrix $A \in \R^{m\times n}$ with $m \ll n$, let the noiseless measurements be $y = A x^*$ and let $H_{0}, H_1,$ denote the distributions over $\R^m$ induced by the projection of $P_{0}, P_1,$ by $A$.
  
  Given $ A , y$, we sample the reconstruction ($\xhat$) according to
  the posterior density 
  \[
  p( \xhat | y) = c_{y} p_{0} ( \xhat | y) + (1-c_{y}) p_{\xtilde} ( \xhat | y),
  \]
  where $c_{y}$ is the posterior probability that $y$ is a projection of $x^*$ drawn from the $P_{0}$ component of $P$. Note that  $c_{y}$ depends on $y$.
  
\begin{figure}[tb]
  \begin{center}
  \includegraphics[scale=0.5]{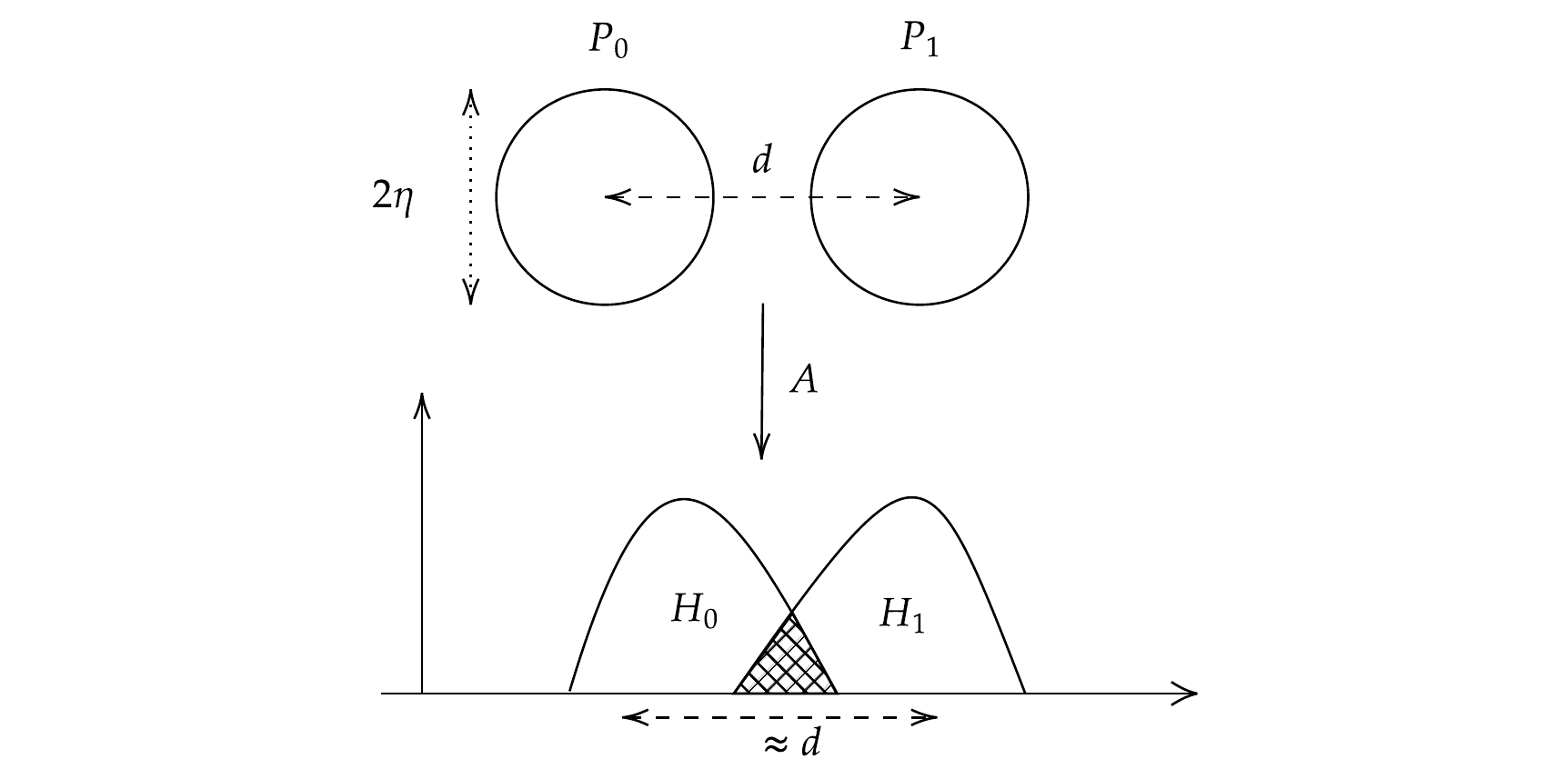}
  \end{center}

  \caption{\small Illustrative example for the upper bound. The signal
          $x^*$ is drawn from a mixture of two well-separated balls.
          The observations $y = Ax^*$ are then drawn from a mixture of
          two distributions $H_0, H_1$ that may overlap.  The
          probability that posterior sampling outputs something
          from the wrong ball is proportional to the (shaded) overlap
          between these distributions, which is atmost
          $1 - TV(H_0, H_1)$.}  \label{fig:two ball example}
\end{figure}

  Since the balls $B_0 \& B_{\tilde{x}}$ are well separated, the
  ground truth and the reconstruction are far apart if and only if
  they lie in different balls, i.e.,
  $ \{x^* \in B_{0}, \xhat \in B_{\xtilde}\},$ or vice versa. It turns
  out quite generally that the probability of this event is bounded by how similar the
  distributions $H_{0}, H_1$ are:

\define{lemma: mixture tv}{Lemma}{
  For $c\in [0,1]$, let $H:= (1-c)H_0 + cH_1$ be a mixture of two
  absolutely continuous distributions $H_0,H_1$ admitting densities
  $h_0, h_1$.  Let $y$ be a sample from the distribution $H$,  such
  that  $y|z^* \sim H_{z^*}$ where $z^*\sim Bernoulli(c)$.
  
  Define $\wh{c}_{y} = \frac{ ch_1(y)}{ (1-c) h_0(y) + ch_1(y)}$, and
  let $\wh{z}|y \sim Bernoulli(\wh{c}_{y})$ be the posterior sampling of $z^*$ given $y$.
  Then we have 
  \begin{equation*}
  \Pr_{z^*, y, \wh{z}}[z^* = 0, \wh{z}=1] \leq 1 - TV(H_0, H_1).
  \end{equation*}
}
\state{lemma: mixture tv}

The proof of this, as well as all parts of the upper bound, can be
found in Appendix~\ref{app:upper}.

In our current example, this gives us
  \begin{align*}
    \Pr[ x^* \in B_{0}, \xhat \in B_{\xtilde}] &\leq 1 - TV(H_{0}, H_1) \text{ and  } \\
    \Pr[ x^* \in B_{\xtilde}, \xhat \in B_{0}] &\leq 1 - TV(H_{0}, H_1).
  \end{align*}

  Since $B_{0}$ and $B_{\xtilde}$ are balls of radius $\eta$, a union bound of the above two probabilities gives:
  \begin{align}
    \Pr\left[ \norm{x^* - \xhat} > 2 \eta \right] 
    \leq & \Pr\left[ x^* \in B_{0}, \xhat \in B_{\xtilde} \right] + \nonumber \\
    & \Pr\left[ x^* \in B_{\xtilde} , \xhat \in B_{0} \right], \nonumber\\
    \leq & 2 \left( 1 - TV\left(H_{0}, H_1 \right) \right).~\label{eqn:TV_ub} 
  \end{align}

  If $A$ is a Gaussian random matrix, the Johnson-Lindenstrauss (JL)
  Lemma tells us that it will preserve distances between vectors with
  high probability .  This does not necessarily mean that every point
  in the distribution $P$ will be preserved in norm.  Still, we show
  that, since $P_0$ and $P_1$ have well-separated supports, their
  projected distributions $H_0$ \& $H_1$ have very high TV distance.
  This also holds more generally, between any distribution on a ball
  and any distribution far from the ball and in the presence of noise.

  \define{lemma: noisy pout tv improved}{Lemma}{
    Let $y$ be generated from $x^*$ by a Gaussian measurement process
    with noise level $\sigma$. For a fixed $\xtilde \in \R^n,$ and
    parameters $\eta>0 , c \geq 4e^2$, let $P_{out}$ be a distribution
    supported on the set 
    \[
    S_{\xtilde, out} := \{ x\in \R^n: \norm{x - \xtilde} \geq c(\eta + \sigma )  \}.
    \]
    Let $P_{\xtilde} $ be a distribution which is supported within an
    $\eta-$radius ball centered at $\xtilde$.
    
    For a fixed $A$, let $H_{\tilde{x}}$ denote the distribution of
    $y$ when $x^* \sim P_{\tilde{x}}.$ Let $H_{out}$ denote the
    corresponding distribution of $y$ when $x^* \sim P_{out}.$    
    Then we have:
    \begin{align*}
    \E_A \left[ TV( H_{\xtilde}, H_{out})  \right] \geq 1 - 4
    e^{-\frac{m}{2}\log\left( \frac{c}{4e^2} \right)} .
    \end{align*}
  }
  \state{lemma: noisy pout tv improved}
  By Markov's inequality, the expectation bound also gives a high probability bound over $A$.

  For our current example, the above result implies that with
  probability $ 1 - e^{-\Omega(m )}$ over $A$, we have
  \begin{align}\label{eqn:TV_lb} 
  TV( H_{0}, H_1) & \geq 1 - e^{ - \Omega( m )}.
  \end{align}

  Substituting equation~\eqref{eqn:TV_lb} in equation~\eqref{eqn:TV_ub}, we have
  \begin{align*}
    \Pr\left[ \norm{x^* - \xhat} > 2 \eta \right] & \leq 2 e^{-\Omega\left( m \right)}.
  \end{align*}
  This shows that posterior sampling will produce a reconstruction
  which is close to the ground truth with overwhelmingly high
  probability for the two-ball example.

  \subsection{Going beyond two balls}
  The two-ball example leaves three main questions unanswered:
  \begin{enumerate}[leftmargin=*]
  \item How do we handle distributions over larger collections of balls?
  \item How do we handle mismatch between the distribution of reality
    ($R$) and the model ($P$)?
  \item How do we handle having a $\delta$ probability of lying outside any ball?
  \end{enumerate}

  \paragraph{Unions of many balls.}  The first question is relatively
  easy to answer: if $\cov_{\eta, 0}(R) \leq e^{o(m)}$, you can cover
  $R$ with a small number of balls, and essentially apply
  Lemma~\ref{lemma: noisy pout tv improved} with a union bound.  There
  are a few details (e.g., Lemma~\ref{lemma: noisy pout tv improved}
  shows you will not confuse any ball with faraway balls, but you
  might confuse it with nearby balls) but solving them is
  straightforward.  This shows that, if $P = R$ and
  $\log \cov_{\eta,0}(R)$ is bounded, then posterior sampling
  works well with $1 - e^{-\Omega(m)}$ probability.

    \define{lemma: wk implies winf}{Lemma}{%
      Let $R,P$ be arbitrary distributions on $\R^n$.
      Let $p \geq 1$ and $\eta, \rho, \delta>0,$ be parameters.

      If $\cW_p(R,P) \leq \rho$ and $\min \{ \logcov_{\eta,\delta}(P), \logcov_{\eta, \delta}(R) \} \leq k$, then there exist distributions $R', R'', P', P'',$ and a finite discrete distirbution $Q$ with $|\supp(Q)| \leq e^{k}$ satisfying:
      \begin{enumerate}
      \item $\min \left\{ \cW_{\infty}(P',Q), \cW_{\infty}(R',Q)  \right\} \leq \eta$,
      \item $\cW_{\infty}(R',P') \leq \frac{\rho}{\delta^{1/p}}$,
      \item $P = (1-2\delta)P' + (2\delta) P''$ and $R = (1-2\delta)R' + (2\delta) R''$
      \end{enumerate}
    }

    \paragraph{Distribution mismatch in $\cW_\infty$.}
    The above assumes we resample with respect to the true
    distribution $R$.  But we only have a learned estimate $P$ of $R$.
    We would like to show that observing samples from $R$ and
    resampling according to $P$ gives good results.  We first show
    that resampling signals drawn from $R$ with respect to $P$ is not
    much worse than resampling signals drawn from $P$ with respect to
    $P$, if $P$ and $R$ are close in $\cW_\infty$.

    \define{lemma: rp winf}{Lemma}{%
      Let $R,P,$ denote arbitrary distributions over $\R^n$ such
      that $\cW_\infty(R,P) \leq \varepsilon$.

      Let $x^* \sim R$ and $z^* \sim P$ and let $y$ and $u$ be
      generated from $x^*$ and $z^*$ via a Gaussian measurement
      process with $m$ measurements and noise level $\sigma$. Let
      $\xhat \sim P( \cdot | y, A)$ and $\zhat \sim P( \cdot | u, A)$.
      For any $d>0$, we have
      \begin{align*}
        &\Pr_{x^* , A , \xi, \xhat } \left[ \norm{x^* - \xhat} \geq d
        + \varepsilon \right] \leq \\
        & e^{-\Omega(m)} + e^{\left( \frac{4\varepsilon\left(
        \varepsilon + 2\sigma \right)m}{2\sigma^2} \right)}\Pr_{z^*, A
        , \xi, \zhat }\left[ \norm{ z^* - \zhat} \geq d \right].
      \end{align*}  
    }
    \state{lemma: rp winf}

    The idea is that with $\sigma$ Gaussian noise, measurements of a
    signal from $R$ aren't too different in distribution from
    measurements of the corresponding nearby signal from $P$.

    Now, if $\cW_\infty(R, P) \ll \sigma$, we would be nearly done:
    Lemma~\ref{lemma: rp winf} says the situation is within $e^{o(m)}$ of the $R = P$ case,
    which we already know gives accurate recovery with
    $O(\log \cov_{\eta, 0}(P))$ measurements.

    \paragraph{Residual mass.}
    There are just two main issues remaining: we want to depend on
    $\log \cov_{\eta, \delta}$ rather than $\log \cov_{\eta, 0}$, and
    we only want to require a bound on $\cW_{1}(R, P)$ not
    $\cW_{\infty}(R, P)$.  By Markov's inequality, these issues are
    very similar: we want to allow both $R$ and $P$ to have a small
    constant probability of behaving badly. To address this, we note
    the existence of two distributions $R'$ and $P'$, which are only
    $\delta$-far in TV from $R$ and $P$ respectively, such that $R'$
    and $P'$ do have a small cover \& are close in $\cW_\infty$.  We
    show that, because posterior sampling would work with $R'$ and
    $P'$, it also works with $R$ and $P$.  This leads to our full
    upper bound:

    \define{thm: main}{Theorem}{%
      Let $\delta \in [0,1/4)$, $p \geq 1$, and $\varepsilon,\eta >0$ be
      parameters.  Let $R,P$ be arbitrary distributions over $\R^n$
      satisfying $\cW_p(R, P) \leq \eps$.  

      Let $x^* \sim R$ and suppose $y$ is generated by a Gaussian
      measurement process from $x^*$ with noise level  $\sigma \gtrsim
      \eps/\delta^{1/p}$ and $m\geq O(\min(\log \cov_{\eta,
      \delta}(R), \log \cov_{\eta, \delta}(P)))$ measurements. Given
      $y$ and the fixed matrix $A,$ let $\xhat$ be the output of
      posterior sampling with respect to $P$. 

  Then there exists a universal constant $c > 0$ such that with probability at least $ 1 - e^{-\Omega(m)}$ over $A, \xi$,
  \begin{align*}
    \Pr_{x^*\sim R , \xhat \sim P( \cdot | y)} \left[ \norm{x^* - \xhat} \geq c \eta + c \sigma \right] &  \leq 2\delta + 2e^{-\Omega(m)}.
  \end{align*}
}%
\state{thm: main}%

Note that we can get a high-probability result by setting
$p = \infty$: if $m \geq O(\log \cov_{\eta, 0}(R))$ and
$\cW_{\infty}(R, P) \leq \sigma$, the error is $O(\sigma + \eta)$ with
$1 - e^{-\Omega(m)}$ probability.

\section{Lower Bound}

In the previous section, we showed, for any distribution $R$ of
signals, that $O(\log \cov(R))$ measurements suffice for posterior
sampling to recover most signals well.  Now we show the converse:
for any distribution of signals $R$, any algorithm for recovery must
use $\Omega(\log \cov(R))$ measurements.

\define{thm: lower}{Theorem}{
  Let $R$ be a distribution supported on a ball of radius $r$ in
  $\R^n$, and $x^* \sim R$.  Let $y = A x^* + \xi$, where $A$ is any
  matrix, and $\xi \sim \cN(0,\frac{\sigma^2}{m}I_m)$.  Assuming
  $\delta<0.1$, if there exists a recovery scheme that uses $y$ and
  $A$ as inputs and guarantees 
	\begin{align*} 
	  \norm{ \hat{x} - x^*} \leq O(\eta),
	\end{align*}
  with probability $\geq 1-\delta$, then we have
	\begin{align*}
    m &\geq \tfrac{0.15}{ \log \left( 1 + \tfrac{m r^2 
    \norm{A}_\infty^2}{\sigma^2} \right)} \left(\logcov_{3\eta, 4
    \delta} (R) + \log 6\delta - O(1)\right).
	\end{align*}

  If $A$ is an i.i.d. Gaussian matrix where each element is drawn from
  $\cN(0,1/m)$, then the above bound can be improved to:
	\begin{align*}
    m \geq \frac{0.15}{ \log \left( 1 + \tfrac{r^2}{\sigma^2} \right)}
      \left(\logcov_{3\eta, 4 \delta} (R) + \log 6\delta -
      O(1)\right).
	\end{align*}
}
\state{thm: lower}

This Theorem is proven using information theory, as an almost direct consequence
of the following three Lemmas. 

First, the measurement process reveals
a limited amount of information:
\define{lemma: MI}{Lemma}{
Consider the setting of Theorem~\eqref{thm: lower}. 
  If $A$ is a deterministic matrix, we have 
  \begin{align*}
    I(y ; x^* ) \leq \frac{m}{2} \log\left( 1 + \frac{m r^2
    \norm{A}_\infty^2}{\sigma^2} \right).
  \end{align*}
  If $A$ is a Gaussian matrix, then
  \begin{align*}
     I(y ; x^* | A) \leq \tfrac{m}{2} \log\left( 1 +
    \tfrac{r^2}{\sigma^2} \right).
  \end{align*}
}
\state{lemma: MI}

Second, since $x^* \to y \to \xhat$ is a Markov chain, we can directly
apply the Data Processing Inequality~\cite{cover2012elements}.
\define{lemma: DPI}{Lemma}{
  Consider the setting of Theorem~\eqref{thm: lower}. 
  If $A$ is a deterministic matrix, we have 
  \begin{align*}
    I(x^*; \xhat) \leq I(y ; x^*).
  \end{align*}

  If $A$ is a random matrix, then 
  \begin{align*}
    I(x^*; \xhat) \leq I(y ; x^* | A) .
  \end{align*}
}
\state{lemma: DPI}

Finally, successful recovery must yield a large amount of information:

\define[Fano variant]{lem: covering continuous}{Lemma}{
  Let $(x, \wh{x})$ be jointly distributed over $\R^n\times \R^n$, where $x \sim R$ and
  $\xhat$ satisfies
  \[
    \Pr[\norm{x - \xhat} \leq \eta] \geq 1-\delta.
  \]
  Then for any $\tau \leq 1 - 3\delta, \delta < 1/3$, we have
  \[
    0.99 \tau (1 - 2\delta)\logcov_{3\eta , \tau + 3\delta }(R)
    \leq I(x; \widehat{x}) + 1.98.
  \]
}%
\state{lem: covering continuous}

In order to complete the proof of Theorem~\ref{thm: lower}, we need an
additional counting argument to remove the extra $\tau$ term that
appears in the left hand side of Lemma~\ref{lem: covering
continuous}.

The proofs can be found in Appendix~\ref{app:lower}.

\section{Experiments}

\begin{figure*}
  \begin{subfigure}[b]{0.48\textwidth}
  \begin{center}
		\includegraphics[width=0.98\columnwidth]{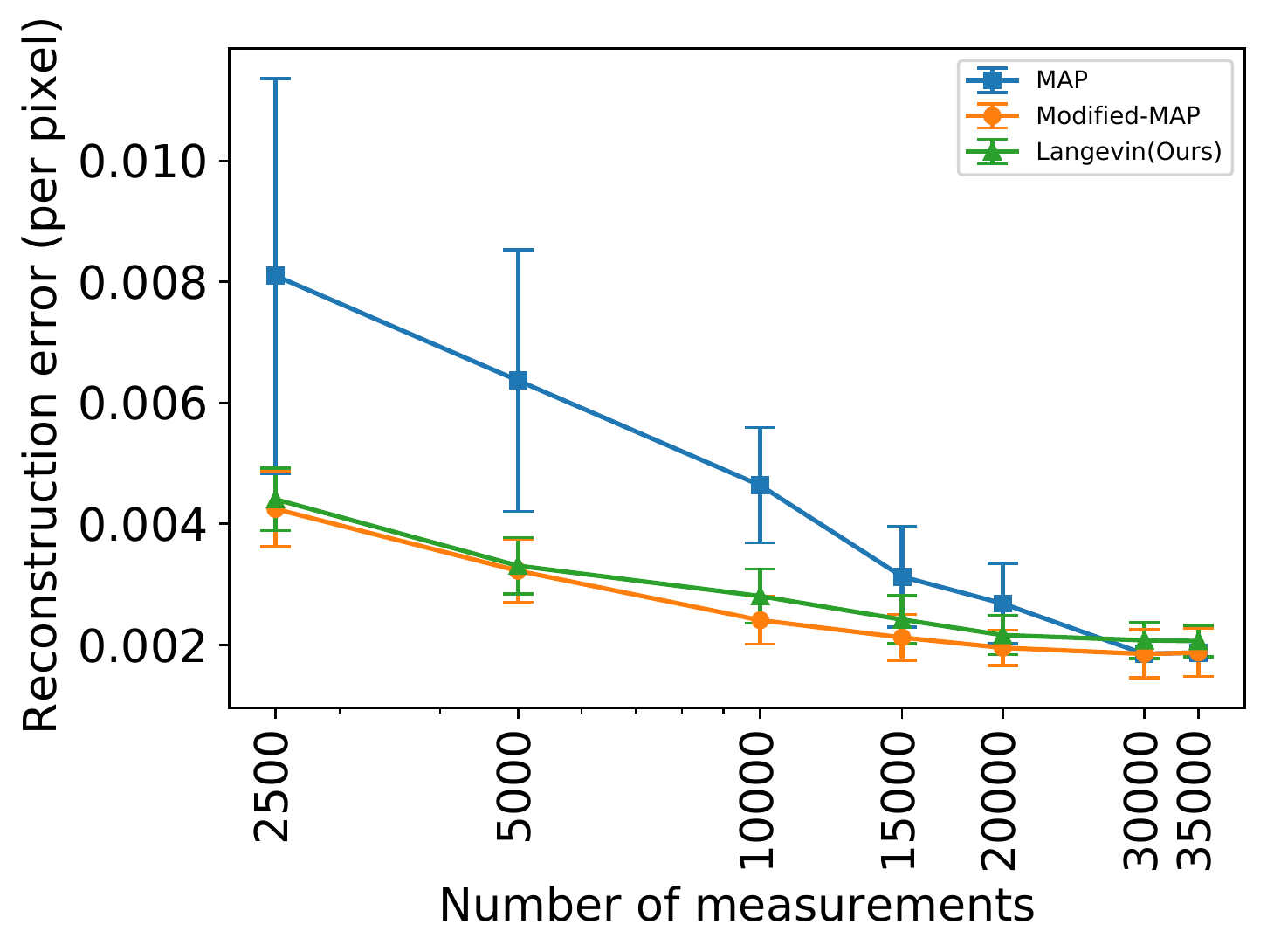}
  \end{center}
  \caption{\small $\norm{x^* - \xhat}^2/n$}
	\label{fig:celeba-l2}
  \end{subfigure}
  \hfill 
  \begin{subfigure}[b]{0.48\textwidth}
  \begin{center}
    \includegraphics[width=0.98\columnwidth]{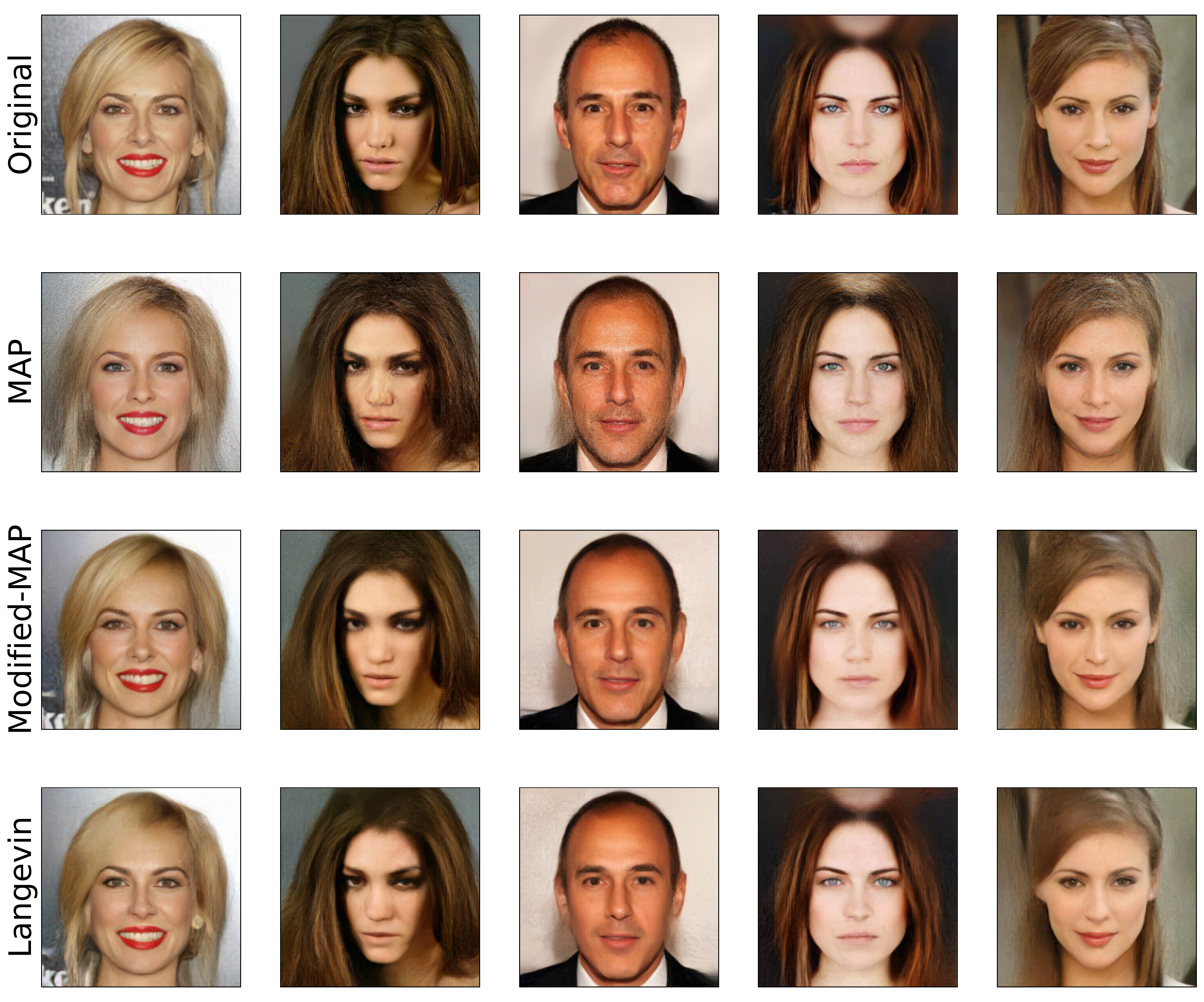}
  \end{center}
  \caption{Reconstructions for $m=20,000$ measurements.}
  \label{fig:celeba-reconstr}
  \end{subfigure}
  \caption{\small We compare our algorithm with the MAP baseline on
    the CelebA-HQ dataset, where the number of pixels is $n=256\times
    256\times 3=196,608.$ In Figure (a) we show a plot of the
    per-pixel reconstruction error as we vary the number of
    measurements $m$. In Figure (b) we show reconstructions obtained
    by each algorithm for $m=20,000$ measurements. We show original
    images (top row), reconstructions by MAP (second row),
    Modified-MAP (third row), and Langevin dynamics (ours, bottom
    row).  Note that MAP produces several artefacts that are not seen
    in Modified-MAP or Langevin dynamics. In these experiments,
    modified-MAP picks hyperparameters based on the reconstruction
    error evaluated on some validation images, while MAP and Langevin
    dynamics pick hyperparameters that maximize the posterior
    likelihood. Here MAP, modified-MAP, and Langevin dynamics all use
  the same Glow model.}
  \label{fig:celeba}
\end{figure*}
  In this section we discuss our algorithm for posterior sampling, discuss why existing algorithms can fail, and show our empirical evaluation of posterior sampling versus baselines.
  
	\subsection{Datasets and Models} We perform our experiments on the
	CelebA-HQ~\cite{liu2018large,karras2017progressive} and
	FlickrFaces-HQ~\cite{karras2019style} datasets.  For the CelebA
	dataset, we run experiments using a Glow generative
	model~\cite{kingma2018glow}. For the FlickrFaces-HQ dataset, we use
	the NCSNv2 model~\cite{song2020improved}.  Both models have output
	size $256\times 256\times 3$.  Details about our experiments are in
	Appendix~\ref{sec: appendix experiments}.  

	\subsection{Langevin Dynamics}
	\paragraph{Glow trained on CelebA-HQ}
  We first consider the Glow generative model, whose distribution $P$
  is induced by the random variable $G(z)$, where $G:\R^n\to\R^n$ is a
  fixed deterministic generative model, and $z\sim \cN(0,I_n)$ .
  Sampling from $p(z|y)$ is easier than sampling from $p(x|y)$, since
  it is easier to compute and we observe that sampling mixes quicker.
  Note that sampling $\wh{z} \sim p(z|y)$ and setting $\wh{x} =
  G(\wh{z})$ is equivalent to sampling $\wh{x} \sim p(x|y)$.

	In order to sample from $p(z|y)$, we use \emph{Langevin dynamics},
	which samples from a given distribution by moving a random initial
	sample along a vector field given by the distribution.  Langevin
	dynamics tells us that if we sample $z_0 \sim \cN(0,1)$, and run the
	following iterative procedure:
  \begin{align*}
    z_{t+1} \leftarrow z_t + \frac{\alpha_t}{2}\nabla_z \log p\left( z_t |y \right) + \sqrt{\alpha_t} \zeta_t,\quad \zeta_t \sim \cN(0,I),
  \end{align*}
  then $p(z|y)$ is the stationary distribution of $z_t$ as $t \to
  \infty$ and $\alpha_t \to 0$.  Unfortunately, this algorithm is slow
  to mix, as observed in \cite{song2019generative}. We instead use an
  annealed version of the algorithm, where in step $t$ we pretend that
  $p(z \mid y)$ has noise scale $\sigma_t \geq \sigma$ instead of
  $\sigma$.  This gives
  \begin{align}\label{eqn:langevin-glow-obj}
   \log p_t(z|y) &=  \left(- \frac{\norm{y - A G(z)}^2}{2\sigma^2_t/m} - \frac{\norm{z}^2}{2}  \right) + \log c(y),
  \end{align}
  where $c(y)$ is a constant that depends only on $y$.  Since we only
  care about the gradient of $\log p(z|y)$, we can ignore this
  constant $c(y)$.  By taking a decreasing sequence of $\sigma_t$ that
  approach the true value of $\sigma$, we can anneal Langevin dynamics
  and sample from $p(z|y)$.  Please refer to Appendix~\ref{sec:
  appendix experiments} for more details about how $\sigma_t$ varies.
	
	\paragraph{NCSNv2 trained on FFHQ}
	We also consider the NCSNv2 model, which takes as input
	the image $x$, and outputs $\nabla_x \log p(x).$ This model is
	designed such that sampling from its marginal involves
	running Langevin dynamics. Since we have access to $\nabla_x \log
	p(x),$ and if we know the functional form of $p(y|x)$, we can easily
	compute $\nabla_x \log p(x|y),$ and run Langevin dynamics via
	\begin{align*}
    x_{t+1} \leftarrow x_t + \frac{\alpha_t}{2}\nabla_x \log p\left(
		x_t |y \right) + \sqrt{\alpha_t} \zeta_t,\; \zeta_t \sim \cN(0,I).
  \end{align*}

	Notice that we can also run MAP using this model. This can be
	achieved by simply following the gradient, and not adding noise:
	\begin{align*}
    x_{t+1} \leftarrow x_t + \frac{\alpha_t}{2}\nabla_x \log p\left(
 x_t |y \right).
  \end{align*}

	This model also requires annealing, and we follow the schedule
	prescribed by~\cite{song2020improved}. Please see Appendix~\ref{sec:
	appendix experiments} for more details.

	\subsection{MAP and Modified-MAP}
	The most relevant baseline for our algorithm is MAP, which was shown
	to be state-of-the-art for compressed sensing using generative
	priors~\cite{asim2019invertible}. 

	Given access to a generative model
  $G$ such that the image $x = G(z)$, and $q(z)$ is the prior of $z$,
  the MAP estimate is
	\begin{align}
		\wh{z} &:= \argmin_{z} \frac{\norm{y - AG(z)}^2}{2\sigma^2/m} - \log
		q(z),
		\label{eqn:map-defn}
	\end{align}
	and set the estimate to be $\wh{x} = G(\wh{z}).$ Typically, $q(z)$ 
	is a standard Gaussian for many generative models. If one has access
  to $p(x)$, such as in NCSNv2~\cite{song2019surfing}, it is possible
  to also do MAP in $x$-space.

	One may modify this algorithm and introduce
	hyperparameters for better reconstructions. We call such
	algorithms \emph{modified-MAP}. For
	example,~\cite{asim2019invertible} introduce a parameter $\gamma >0$
	that weights the prior, and their estimate is
	\begin{align}
		\wh{z}_{modified} := \argmin_{z} \norm{y -
		AG(z)}^2 -\gamma \log
		q(z),
		\label{eqn:modified-map-defn}
	\end{align}
	Other examples of hyper-parameters include early stopping to avoid
	``over-fitting'' to the measurements, and choosing optimization 
	parameters such that the reconstruction error is minimized on a validation set of
	images. Then these hyper-parameters are used for evaluating reconstruction error on 
	a different test.

\begin{figure*}[t]
  \begin{subfigure}[t]{0.48\textwidth}
  \begin{center}
		\includegraphics[width=0.98\columnwidth]{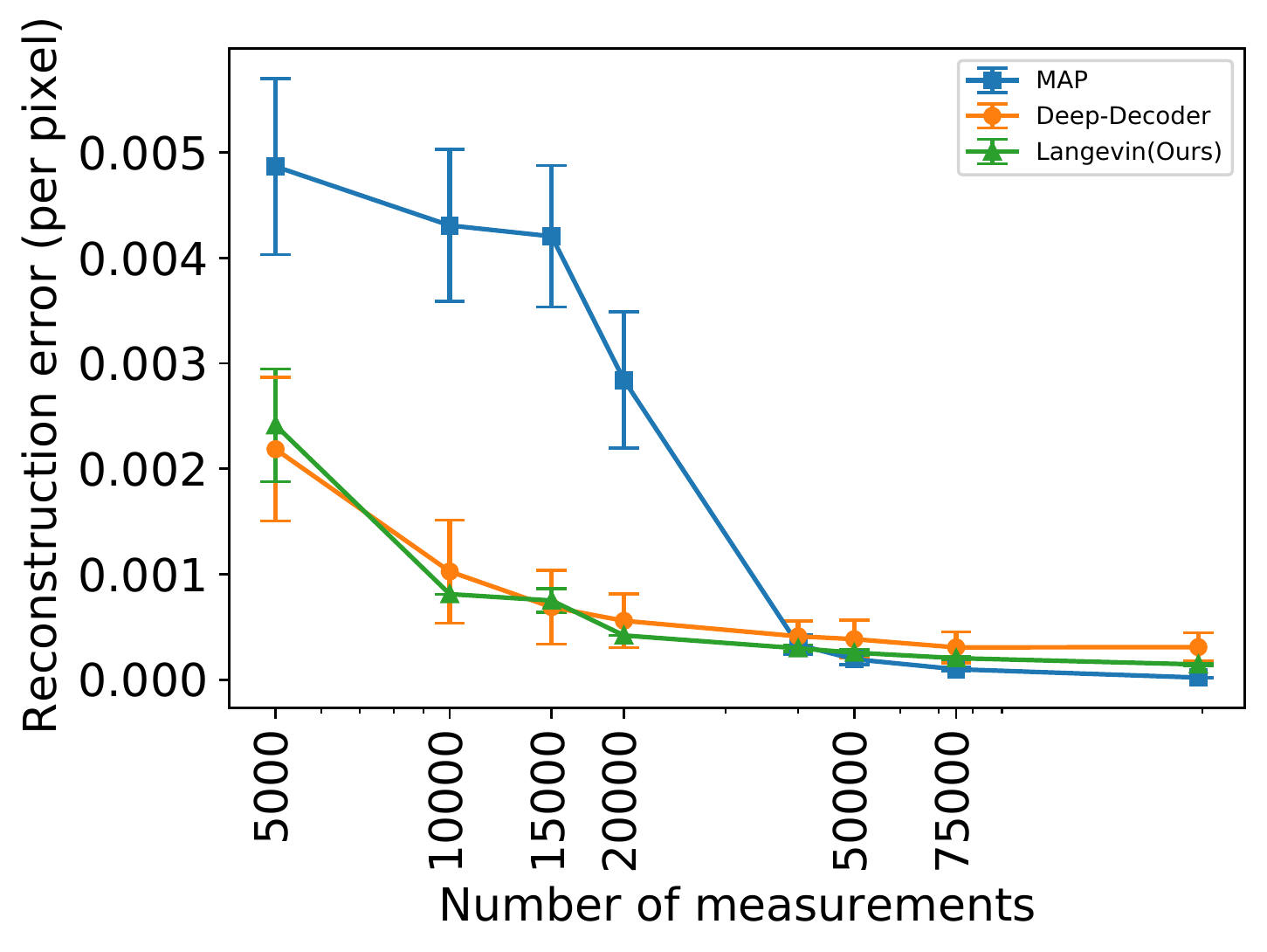}
  \end{center}
  \caption{\small $\norm{x^* - \wh{x}}^2/n$}
	\label{fig:ffhq-l2}
  \end{subfigure}
  \hfill
  \begin{subfigure}[t]{0.48\textwidth}
  \begin{center}
		\includegraphics[width=0.98\columnwidth]{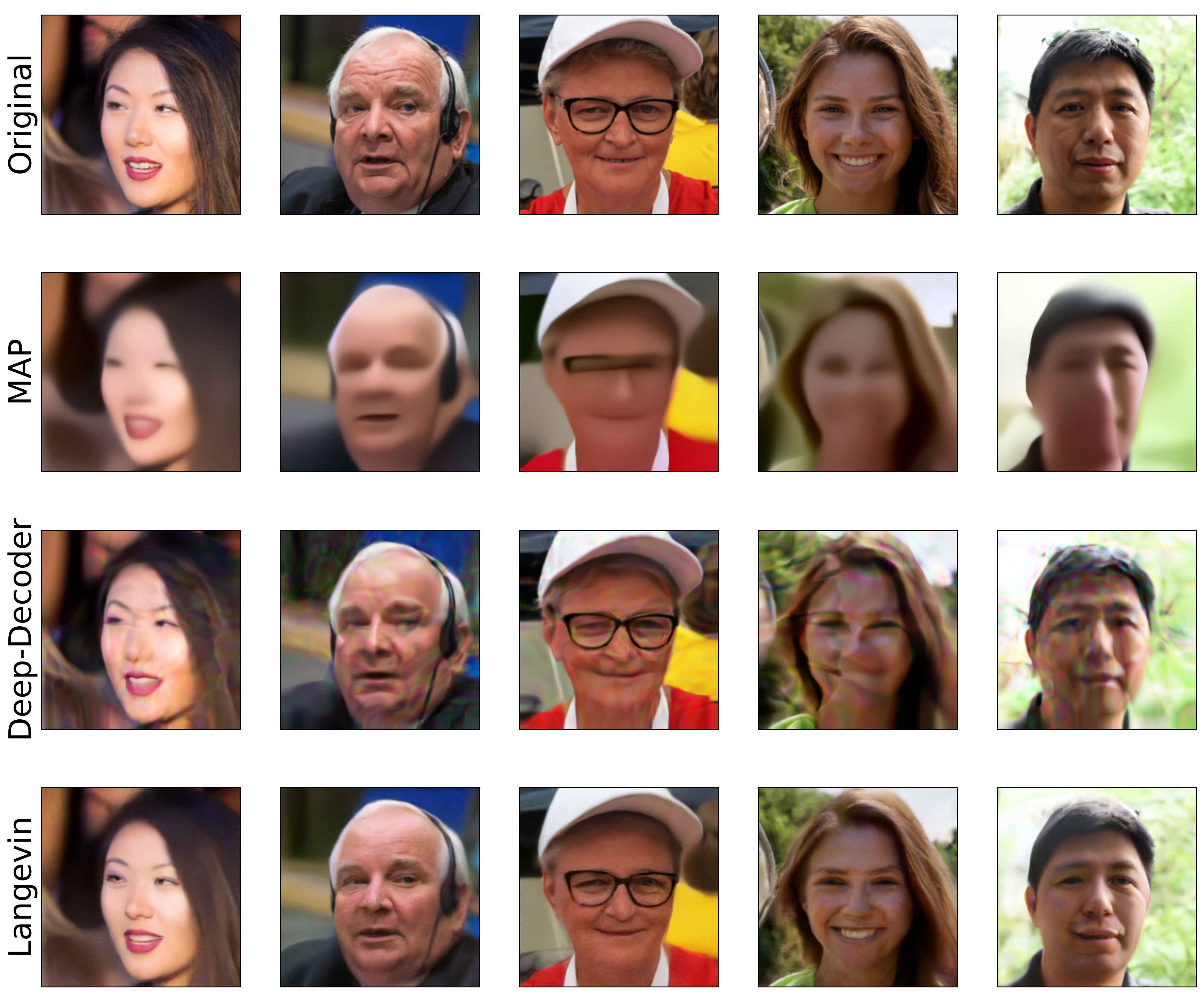}
  \end{center}
	\caption{Reconstructions for $m=5,000$ measurements.}
	\label{fig:ffhq-reconstr}
  \end{subfigure}
  \caption{\small We compare our algorithm with the MAP and
    Deep-Decoder baselines on the FFHQ dataset, where the number of
    pixels is $n=256\times 256\times 3=196,608.$ Figure (a) plots
    per-pixel reconstruction error as we vary the number of
    measurements $m$.  Figure (b) shows original images (top row),
    reconstructions by MAP (second row), Deep-Decoder (third row), and
    Langevin dynamics (bottom row). Langevin dynamics is the practical
    implementation of our proposed posterior sampling estimator.
    Note that although Deep Decoder and Langevin achieve similar value
    of reconstruction errors, Langevin produces images with higher
    perceptual quality, as can be seen in Figure (b).}
  \label{fig:ffhq}
\end{figure*}

\subsection{Experimental Results}
  MAP estimation does not work on general distributions: as an extreme
  example, if $R$ is a mixture of some continuous distribution 99\% of
  the time, and the all-zero image 1\% of the time, it will always
  output the all-zero image, which is wrong 99\% of the time.  More
  generally, looking for high-likelihood \emph{points} rather than
  \emph{regions} means it prefers sharp but very narrow maxima to
  wide, but slightly shorter, maxima.  Posterior sampling prefers
  the opposite. We now study this empirically.

  \paragraph{CelebA.} In Figure~\ref{fig:celeba}, we show the
  performance of our proposed algorithm for compressed sensing on
  CelebA-HQ with Glow. The baselines we consider are MAP, and
  modified-MAP. MAP directly optimizes the objective defined in
  Eqn~\eqref{eqn:map-defn} while Modified-MAP
  optimizes~\eqref{eqn:modified-map-defn}. The MAP baseline in
  Figure~\ref{fig:celeba} tries to maximize the posterior likelihood,
  and hence hyperparameters are selected so that the posterior is
  optimized. In contrast, what we term the modified-MAP algorithm was
  proposed by~\cite{asim2019invertible}, and this algorithm picks
  hyperparameters that minimize reconstruction error on a holdout set
  of images.  These hyperparameters are significantly worse at
  optimizing the MAP objective, but lead to more accurate recovered
  images, presumably due to some sort of implicit regularization.
  This modified-MAP method has shown to be state-of-the-art for
  compressed sensing on CelebA~\cite{asim2019invertible}.

	We find that our algorithm is competitive with respect to
	modified-MAP, and beats MAP when the measurements are $<35,000.$ 

	\paragraph{FFHQ.}
  In Figure~\ref{fig:ffhq}, we show the performance of our proposed
  algorithm for compressed sensing on FlickrFaces-HQ with the NCSNv2
  generative model. We consider MAP and
  Deep-Decoder~\cite{heckel2018deep} as the baselines.  Note that the
  NCSNv2 model was designed for Langevin dynamics, and we adapt it to
  MAP. Hence, we choose the Deep-Decoder as a second baseline, as it
  has been shown to match state-of-the-art~\cite{asim2019invertible}.

	We observe that for $m<40,000$ measurements, Langevin dynamics beats
	MAP, and is competitive with Deep-Decoder. In
	Figure~\ref{fig:cs-reconstr-ffhq} we visually compare the reconstruction
	quality as the number of measurements increases. Note that although
	Langevin and Deep-Decoder have similar reconstruction errors in
	Fig~\ref{fig:ffhq-l2}, the images in Fig~\ref{fig:cs-reconstr-ffhq}
	produced by Langevin dynamics have better perceptual quality. Also
	see Fig~\ref{fig:ffhq-reconstr} for more examples of reconstructions
	at $m=5,000$ measurements.
	
	\paragraph{Inpainting.} In order to highlight the difference in
	diversity between images produced by MAP and Langevin dynamics, we
	evaluate them on the inverse problem of inpainting missing pixels.
	As shown in Figure~\ref{fig:inpaint-hair}, when the hair and
	background of a ground truth image is removed, MAP produces a single
	``most likely'' reconstruction, while Langevin produces diverse
	images that satisfy the measurements. Each column for Langevin
	dynamics in Figure~\ref{fig:inpaint-hair} corresponds to a run
	starting from a random initial point. We do not observe any change
	in MAP reconstructions as we vary the initial point.  
	
	We believe that the MAP reconstruction, while in some sense a highly
	likely reconstruction, is abnormally ``washed out'' and indistinct;
	analogous to how zero is the most likely sample from $N(0, I_d)$,
	yet is extremely atypical of the distribution.  We see this
	quantitatively in that the corresponding $\|z\|^2/n$ for MAP is
	$0.007$, even though samples from $R$ almost surely have $\|z\|^2/n
	\approx 1$, as do those of Langevin.

\section{Conclusion}

This paper studies the problem of compressed sensing a signal from a
distribution $R$.  We have shown that the measurement complexity is
closely characterized by the log approximate covering number of $R$.
Moreover, this recovery guarantee can be achieved by posterior
sampling, even with respect to a distribution $P \neq R$ that is close
in Wasserstein distance.  Our experiments using Langevin dynamics to
approximate posterior sampling match state-of-the-art recovery with a
theoretically grounded algorithm.

This measurement complexity is inherent to the true distribution of
images in the domain, and can't be improved.  But perhaps it can be
estimated: one open question is whether $\log \cov_{\eta,\delta}(P)$ can be
estimated or bounded when $P$ is given by a neural network generative
model.

\section{Acknowledgements} 
Ajil Jalal and Alex Dimakis were supported by NSF Grants CCF
1934932, AF 1901292, 2008710, 2019844 the NSF IFML 2019844 award 
and research gifts by Western Digital, WNCG and MLL, computing
resources from TACC and the Archie Straiton Fellowship.  Sushrut
Karmalkar was supported by a University Graduate Fellowship from 
UT Austin. Eric Price was supported by NSF Award
CCF-1751040 (CAREER) and NSF IFML 2019844.

\bibliographystyle{icml2021}
\bibliography{main}

\appendix 
\onecolumn
\section{Upper Bound Proofs}\label{app:upper}
\subsection{Proof of Lemma~\ref{lemma: mixture tv}}\label{sec: proof mixture tv}
\restate{lemma: mixture tv}
\begin{proof}
  We have
  \begin{align}
    \Pr_{z^*, y, \wh{z}}[z^* = 0, \wh{z}=1] &= \Pr[z^* = 0] \E_{y \sim h_0, \wh{z}|y} [ {1}\{ \wh{z} = 1\}],\\
    &= (1 - c) \int h_0 (y) \Pr [ \wh{z} = 1 | y ] dy.
  \end{align}

  By definition, we have 
  \begin{align*}
    \Pr [ \wh{z} = 1 | y ]  &= \frac{ ch_1 (y)}{(1-c) h_0(y) + ch_1(y)}.
  \end{align*}
  
  Substituting, we have 
  \begin{align*}
    \Pr_{z^*, y, \wh{z}}[z^* = 0, \wh{z}=1] &= \int \frac{ (1-c) h_0 (y) ch_1 (y)}{ (1-c) h_0 (y) + ch_1 (y)} dy \\
    &\leq \int \frac{ (1-c) h_0(y) \cdot ch_1 (y)}{\max \{ (1-c) h_0(y), c h_1 (y) \} } dy\\ 
    &= \int \min \{ (1-c) h_0(y) ,  ch_1 (y) \} dy \\
    &\leq \int \min \{ h_0(y) ,  h_1 (y) \} dy \\
    &= (1 - TV(H_0, H_1)).
  \end{align*}

\end{proof}

\subsection{Proof of Lemma~\ref{lemma: noisy pout tv improved}}\label{sec: proof noisy pout tv improved}

\restate{lemma: noisy pout tv improved}
\begin{proof}
  In order to prove the lemma, it suffices to show that on the set
  \begin{align*}
    B := \{ y \in \R^m: \norm{ y - A \tilde{x}} \leq  \sqrt{c} \left( \eta +  \sigma \right)\},
  \end{align*}
  we have
  \begin{align}
    \E_A [H_{out} (B) ]&\leq 2 e^{-\frac{m}{2}\log\left( \frac{c}{4e^2} \right)} \label{eqn: noisy pout tv improved eqn 1},\\
    \E_A [H_{\tilde{x} } (B)] &\geq 1 - 2 e^{-\frac{m}{2}\log\left( \frac{c}{4e^2} \right)} \label{eqn: noisy pout tv improved eqn 2}.
  \end{align}

  Using the above bounds, we can conclude that 
  \[
  \E_A \left[ TV(H_{out}, H_{\tilde{x}} ) \right] \geq \E_A [H_{\tilde{x}}(B)] - \E_A [H_{out}(B)] \geq 1 - 4e^{-\frac{m}{2}\log\left( \frac{c}{4e^2} \right)}.
  \]

  First we prove Equation~\eqref{eqn: noisy pout tv improved eqn 1}.

  Consider the joint distribution of $y, A$.
  We have
  \begin{align}
    \E_A \left[ H_{out}(B)  \right] & =  \E_A \left[ \E_{x \sim P_{out}} \left[ \cN\left(Ax, \frac{\sigma^2}{m} I_m\right)(B)  \right] \right],\\
    & = \E_{x \sim P_{out}} \left[ \E_A \left[ \cN( A x, \sigma^2/m)(B)  \right] \right], \label{eqn: hout eqn 1}
  \end{align}
  where the first line follows from the definition of $H_{out}$ and the fact that $x, A$ are independent. The last line follows by switching the order of integrating $A, x$. Here $\cN( A x, \sigma^2/m)(B)$ refers to the mass $\cN(Ax, \sigma^2/m)$ places on $B$. 

  Consider a fixed $x \in S_{\tilde{x},out}$, that is, $x$ lies in the support of $P_{out}$ and satisfies $\norm{ x - \tilde{x}} \geq c (\eta + \sigma\sqrt{m})$.
  We split the above expectation into two conditions over the matrix $A$.
  \begin{itemize}[leftmargin=*]
    \item Case 1:  $ \norm{A x - A \tilde{x}} \leq 2 \sqrt{c} \left( \eta+\sigma \right)$. Since $A$ is i.i.d. Gaussian, $ A \left( x - \tilde{x} \right)$ is distributed as $\cN \left(0, \frac{\norm{x - \tilde{x}}^2}{m}I_m \right)$.
      This gives 
      \begin{align*}
	\Pr_A \left[ \norm{A x - A \tilde{x}} < 2 \sqrt{c} \left( \eta + \sigma \right) \right] & \leq \Pr_A\left[ \norm{A x - A \tilde{x}} \leq \frac{2}{\sqrt{c}}\norm{x - \tilde{x}} \right] ,\\
	&\leq \frac{2}{\sqrt{m\pi}}\left( \frac{2e}{\sqrt{c}} \right)^m ,\\
	&= \frac{2}{\sqrt{m\pi}} e^{-\frac{m}{2}\log\left( \frac{c}{4e^2} \right)} ,\\
	&\leq e^{-\frac{m}{2}\log\left( \frac{c}{4e^2} \right)} \quad \text{ if } m > 1.
      \end{align*}

      This implies
      \begin{align*}
	\E_{x \sim P_{out}} \left[ \E_A \left[ \cN( A x, \sigma^2/m)(B) {1}_{\norm{A x - A \tilde{x}} < 2 \sqrt{c}\left( \eta + \sigma \right)} \right] \right]
	& \leq \E_{x \sim P_{out}} \left[ \E_A \left[ {1}_{\norm{A x - A \tilde{x}} < 2 \sqrt{c}\left( \eta + \sigma \right)} \right] \right],\\
	& = \E_{x \sim P_{out}} \left[ \Pr_{A }\left[ \norm{A x - A \tilde{x}} \leq 2 \sqrt{c} \left( \eta + \sigma \right) \right]  \right],\\
	&\leq e^{-\frac{m}{2}\log\left( \frac{c}{4e^2}\right)}.
      \end{align*}

    \item Case 2: $ \norm{ A x - A \tilde{x} } > 2\sqrt{c} \left(  \eta + \sigma \right)$.

      Recall the definition of $B := \left\{ y \in \R^m: \norm{ y - A \tilde{x}} \leq \sqrt{c} \left( \eta + \sigma \right)  \right\}$.
      For any $y \in B$, $x$ in the support of $P_{out}$ and for $A$ such that $ \norm{A x - A \tilde{x}} > 2  \sqrt{c} \left( \eta + \sigma \right)$, we have
      \begin{align*}
	\norm{ y - A x} \geq \norm{A x - A \tilde{x}} - \norm{ y - A \tilde{x} } \geq 2\sqrt{c} \left( \eta+ \sigma \right) - \sqrt{c} \left( \eta+\sigma\right) = \sqrt{c} \left( \eta+ \sigma\right).
      \end{align*}
      For each $x$ in the support of $P_{out}$, define the set $B_x := \left\{  y\in \R^m: \norm{ y - A x}  \geq \sqrt{c} \left( \eta+ \sigma\right)\right\}.$
      The above inequality gives $B\subseteq B_x$ for each $x$ in the support of $P_{out}$.
      This gives
      \begin{align*}
	\cN ( Ax, \sigma^2)(B) \leq \cN(Ax, \sigma^2)(B_x) \leq  e^{- 2 \left( \sqrt{c} - 1  \right)^2 m } \leq e^{-\frac{mc}{2}}.
      \end{align*}
      where the last inequality follows by the definition of $B_x$ and Gaussian concentration of $\cN(Ax,\sigma^2)$ on the set $B_x$, and since $ 2\left( \sqrt{c} -1 \right)^2 > \frac{c}{2}$ if $c\geq 4$.

  \end{itemize}

  Substituting the inequalities from Case 1 and Case 2 in Eqn~\eqref{eqn: hout eqn 1}, we have
  \begin{align*}
    \E_A \left[ H_{out}(B)  \right] &= \E_{x \sim P_{out}} \left[ \E_A \left[ \cN( A x, \sigma^2/m)(B)  \right] \right],\\
    & \leq e^{-\frac{m}{2} \log \left( \frac{c}{4e^2} \right)} + e^{-\frac{cm}{2}},\\
    & \leq 2 e^{-\frac{m}{2}\log\left( \frac{c}{4e^2} \right)} \quad \text{ if } c \geq 4 e^2.
  \end{align*}

  This proves Eqn~\eqref{eqn: noisy pout tv improved eqn 1}.

  A similar proof can be used to show that
  \begin{align*}
    \E_A \left[ H_{\tilde{x}}(B^c)  \right] \leq 2 e^{-\frac{m}{2}\log\left( \frac{c}{4e^2} \right)}.
  \end{align*}

  This proves Eqn~\eqref{eqn: noisy pout tv improved eqn 2}.

  Putting the two above inequalities together, we have 
  \[
   \E_A TV(H_{out}, H_{\tilde{x}}) \geq \E_A[ H_{\tilde{x}}(B)] - \E_A[ H_{out(B)}] \geq 1 - 4e^{-\frac{m}{2}\log\left( \frac{c}{4e^2} \right)}.
  \]

  This concludes the proof.
\end{proof}

\subsection{Proof of Lemma~\ref{lemma: wk implies winf}}\label{sec: proof wk implies winf}

\state{lemma: wk implies winf}

\begin{proof}
Since the statement of the lemma is symmetric with respect to $P$ and $R$, WLOG let $\log \cov_{\eta, \delta}(P) \leq k$. Then there is an $S \subset \mathbb{R}^n$ such that $|S| \leq e^k$ and 
  \begin{align*}
  \Pr_{x \sim P}[x \in \cup_{u \in S} B(u, \eta)] = 1-c_P \geq 1-\delta,
  \end{align*}
  
  We define the function $f:\R^n \to \R_+$ as
  \begin{align*}
    f(x) =\begin{cases}
      \frac{1}{|\{u \in S \mid x \in B(u, \eta) \}| } & \text{ if } \exists u \in S \; s.t.\; x \in B(u,\eta), \\
      0 & \text{ otherwise}.
    \end{cases}
  \end{align*}
  By construction, $f$ is a piecewise constant function that is inversely proportional to
  the number of $\eta-$radius balls centered around points in $S$ cover a point $x$. 

  For each $u \in S$, we define the measure $Q''$ as
  \begin{align*}
    Q''(u) := \int_{ B(u, \eta)} f \; dP.
  \end{align*}

  Observe that 
   \begin{align*} 
  \sum_{u \in S} Q''(u) &= \sum_{ u \in S} \int_{B(u, \eta)} f dP, \\
  &= \int_{\cup_{u \in S}B(u, \eta)} dP = 1 - c_P
  \end{align*} 

  Notice that $Q''$ is not a probability distribution, since it only has mass $1-c_P$.
  However we can create a distribution $Q'$ from $Q''$ by putting an additional $c_P$ mass on some arbitrary point in $\mathbb{R}^n$ (say, $0$). By construction, there exists a coupling $\Pi$ of $P$ and $Q'$ where the coupling distributes the mass at each point in $\R^n$ to points $\eta$ close to it in $S$, such that 
     \begin{align}\label{eqn:c_P bound}
  c_P = \Pr_{(x_1, x_2) \sim \Pi} \left[ \norm{x_1 - x_2} \geq \eta \right] \leq\delta.
  \end{align}

  Additionally, since $W_p(R, P) \leq \rho$, there exists a coupling $\Gamma$ such that. 
   \begin{align}\label{eqn:c_R bound}
  c_R = \Pr_{(x_1, x_2) \sim \Gamma} \left[ \norm{x_1 - x_2} \geq \frac{\rho}{\delta^{1/p}} \right] \leq \frac{\E\left[ \norm{x_1 - x_2}^p \right]}{\frac{\rho^p}{\delta}} \leq \delta.
  \end{align}
  where $c_P$ is defined by the first equality.
  We can hence define a couple between $P, Q', R$ whose distribution is given by the following -- for any borel measurable sets $B_1, B_2, B_3$ we have $\Omega(B_1, B_2, B_3) = P(B_1) \Pi(B_2 \mid B_1)\Gamma(B_3 \mid B_1)$. To verify that this is indeed a coupling of the kind we want, we observe that the marginals of $\Omega$ are $P, Q$ and $R$ respectively. 
  \begin{enumerate}
  	\item $\Omega(B_1, \R^n, \R^n) = P(B_1) \Pi(\R^n \mid B_1)\Gamma(\R^n \mid B_1) = P(B_1)$. 
  	\item $\Omega(\R^n, B_2, \R^n) = P(\R^n) \Pi(B_2 \mid \R^n)\Gamma(\R^n \mid \R^n) = 1 \cdot \frac{\Pi(B_2, \R^n)}{P(\R^n)} \cdot 1 = Q'(B_2)$.
  	\item $\Omega(\R^n, \R^n, B_3) = P(\R^n) \Pi(\R^n \mid \R^n)\Gamma(B_3 \mid \R^n) = R(B_3)$.
  \end{enumerate}

  To define $P', Q, R'$, we look at $\Omega$ conditioned on the event $E := \{ (x, y, z) \mid \|x-z\| \leq \rho/\delta^{1/p} \text{ and } \norm{x-y} \leq \eta \}$. To estimate the probability of $E$, we define $E_1 := \{ (x, y, z) \mid z \in \R^n \text{ and } \norm{x-y} > \eta \}$ and $E_2 :=  \{ (x, y, z) \mid \|x-z\| > \rho/\delta^{1/p} \text{ and } y \in \R^n \}$. Then, $\overline{E} = E_1 \lor E_2$. 
  
  We now show that $\Omega(E_1) \leq \delta$. Let $(E_1)_I$ denote $E_1$ restricted to the coordinates in $I$. 
  \begin{align*}
  \Omega(E_1) := P((E_1)_{1}) \Pi((E_1)_{1,2}\mid (E_1)_{1}) \Gamma((E_1)_{1,3} \mid (E_1)_{1}) \leq \Pi((E_1)_{1,2}) \leq \delta,
  \end{align*}
  where the first inequality is because $\Gamma((E_1)_{1,3} \mid (E_1)_{1}) \leq 1$ and $\Pi((E_1)_{1,2}\mid (E_1)_{1}) = \Pi((E_1)_{1,2})/P((E_1)_{1})$ and the final inequaity follows from equation~\eqref{eqn:c_P bound}. The bound for $E_2$ follows similarly. A union bound shows that $\Omega(E) \geq 1-2\delta$. We can restrict the event $E$ further to have mass $1-2\delta$. 
  
   We look at the marginals of the  conditional couple $\Omega(\cdot \mid E)$ to get distributions $P', Q, R'$ as follows. We define $P'(\cdot) := \Omega(\cdot, \R^n, \R^n \mid E)$, $Q(\cdot) := \Omega(\R^n, \cdot, \R^n \mid E)$ and $R'(\cdot) := \Omega(\R^n,\R^n, \cdot \mid E)$. $P''$ and $R''$ are defined similarly via conditioning on $\overline{E}$. Hence, $P(\cdot) = \Omega(\cdot, \R^n, \R^n) = \Omega(E) \Omega(\cdot, \R^n, \R^n \mid E) + \Omega(\overline{E}) \Omega(\cdot, \R^n, \R^n \mid \overline{E}) = (1-2\delta)P'(\cdot) + (2\delta) P''(\cdot)$. The statement for $R$ follows similarly. 

   This finally gives distributions $P', R',Q,$ such that:
   \begin{enumerate}
   	\item $\cW_{\infty}(P', Q) \leq \eta$
   	\item $\cW_{\infty}(R', P') \leq \rho/\delta^{1/p}$ 
   	\item $P = (1-2\delta)P' + (2\delta) P''$ and $R = (1-2\delta)R' + (2\delta) R''$.  
   \end{enumerate}
The first two statements follow because of the event we condition over.

 Note that this restriction does not change the fact that $\supp(Q) < e^k$, and hence we have our result. 

\end{proof}

\subsection{Proof of Lemma~\ref{lemma: rp winf}}\label{sec: proof rp winf}

\restate{lemma: rp winf}
\begin{proof}
  Let $B_1$ denote the event 
  \begin{align*}
    B_1 = \left\{ \norm{x^* - \widehat{x} } \geq d + \varepsilon \right\}.
  \end{align*}

  Similarly, let $B_2$ denote the event 
  \begin{align*}
    B_2 = \left\{ \norm{z^* - \widehat{x} } \geq d \right\}.
  \end{align*}

  We have
  \begin{align*}
    \Pr_{x^* \sim R, A, \xi, \widehat{x} \sim P( \cdot | A, y)} \left[ B_1 \right] &=  \E_{x^* \sim R} \E_{A} \left[ \E_{y| A, x^*}\left[ \E_{\widehat{x} \sim P(\cdot| y, A)} [ {1}_{B_1} ] \right] \right].
  \end{align*}

  We can write the integral over $R$ as an integral over the coupling
  $\Pi$ between $R,P$.
  This gives

  \begin{align*}
    \Pr_{x^*, A, \xi, \widehat{x} \sim P( \cdot | A, y)} \left[ B_1 \right] &=  \E_{x^*,z^*} \E_{A} \left[ \E_{y| A, x^*}\left[ \E_{\widehat{x} \sim P(\cdot| y, A)} [ {1}_{B_1} ] \right] \right].
  \end{align*}

  Since $x^*, z^*$ are coupled and $W_\infty(R,P)\leq \varepsilon$, we have $\norm{ x^* - z^* } \leq \varepsilon$ almost surely.
  This gives $ B_1 \subseteq B_2$ if $x^*, z^*$ are distributed according to $\Pi$.
  Hence,
  \begin{align*}
    \Pr_{x^*, A, \xi, \widehat{x} \sim P( \cdot | A, y)} \left[ B_1 \right] &\leq  \E_{x^*,z^*} \E_{A} \left[ \E_{y| A, x^*}\left[ \E_{\widehat{x} \sim P(\cdot| y, A)} [ {1}_{B_2} ] \right] \right].
  \end{align*}

  We can split the above integral into two parts: one where the matrix $A$ satsifies $\norm{ A x^* - A z^*} \leq 2 \varepsilon$, and another case where $\norm{A x^* - A z^*} > 2\varepsilon$. 
  This gives

  \begin{align}
    \Pr_{x^*, A, \xi, \widehat{x} \sim P( \cdot | A, y)} \left[ B_1 \right] &\leq  \E_{x^*,z^*} \E_{A} \left[ {1}_{\norm{A x^* -A z^*} > 2\varepsilon} \E_{y| A, x^*}\left[ \E_{\widehat{x} \sim P(\cdot| y, A)} [ {1}_{B_2} ] \right] \right]  (*) \label{eqn: rp winf eqn 1}\\ 
    & + \E_{x^*,z^*} \E_{A} \left[ {1}_{\norm{A x^* -A z^*} \leq 2\varepsilon} \E_{y| A, x^*}\left[ \E_{\widehat{x} \sim P(\cdot| y, A)} [ {1}_{B_2} ] \right] \right].(**)\label{eqn: rp winf eqn 2}
  \end{align}

  Consider the term$(*)$ in line~\eqref{eqn: rp winf eqn 1}.
  We have
  \begin{align}
   \E_{x^*,z^*} \E_{A} \left[ {1}_{\norm{A x^* -A z^*} > 2\varepsilon} \E_{y| A, x^*}\left[ \E_{\widehat{x} \sim P(\cdot| y, A)} [ {1}_{B_2} ] \right] \right] & \leq \E_{x^*,z^*} \E_{A} \left[ {1}_{\norm{A x^* -A z^*} > 2\varepsilon} \right], \\  
   & \leq \E_{x^*, z^*} \left[ e^{-\Omega(m)} \right] \leq e^{-\Omega(m)},\label{eqn: rp winf eqn3}
  \end{align}
  where the last inequality follows from the Johnson-Lindenstrauss lemma for a fixed $x^*, z^*$, and hence is true on average over $x^*, z^*$ drawn independent of $A$.

  Now consider the term $(**)$ in line~\eqref{eqn: rp winf eqn 2}.
  Notice that since the noise in the measurements is Gaussian, we have 
  \begin{align*}
    y | x^*, A \sim \cN( Ax^* , \sigma^2 /m ).
  \end{align*}

  We break the integral over $y$ in $(**)$ into two cases:
  \begin{enumerate}

    \item Case 1: $ \norm{y - A x^*} > 2 \sigma$.
      Since $p(y | A , x^*)$ is distributed as $\cN\left( Ax^*, \frac{\sigma^2}{m} I_m \right)$, by standard Gaussian concentration, we have
      \begin{align*}
	\int_{y: \norm{y - A x^*} > 2 \sigma} p(y | A, x^*) dy & \leq e^{-\Omega(m)}.
      \end{align*}

    \item Case 2: $\norm{ y - A x^*} \leq 2\sigma$. This gives
      \begin{align*}
	\norm{ Ax^* - y}^2 &= \norm{A x^* - y}^2 - \norm{ y - Az^*}^2 + \norm{y - Az^*}^2 ,\\ 
	&= \norm{ A x^* - y}^2 - \norm{y - A x^* + A x^* - A z^*}^2 + \norm{ y - A z^*}^2,\\
	&=  -\norm{A x^* - A z^*}^2 - 2 \inner{y - Ax^*, A x^* - Az^*} + \norm{ y - Az^*}^2.
      \end{align*}

      Observe that in $(**)$, we have 
      \begin{align*}
	\norm{A x^* - A z^*} &\leq 2\varepsilon \Rightarrow \norm{A x^* - A z^*}^2 \leq 4\varepsilon^2.
      \end{align*}
      By the Cauchy-Schwartz inequality and the assumption that $\norm{y - A x^*}\leq 2\sigma$, we have 
      \begin{align*}
        2\inner{y - A x^* , A x^* - A z^*} &\leq 8\sigma\varepsilon.
      \end{align*}

      Substituting the above two inequalities, we have
      \begin{align}
	\norm{ Ax^* - y}^2 &\geq  -4 \varepsilon^2 -  8\sigma\varepsilon  + \norm{ y - Az^*}^2 , \label{eqn: rp winf eqn 3}\\ 
	\Rightarrow \exp{\left( -\frac{\norm{ A x^* - y}^2}{2\sigma^2/m} \right)}&\leq  \exp{ \left( \frac{4\varepsilon\left( \varepsilon + 2\sigma \right) m}{2\sigma^2} \right) } \exp{\left( -\frac{\norm{ A z^* - y}^2}{2\sigma^2/m} \right)},\\
      \end{align}
      Observe that the LHS has the density of measurements from $x^*$, while the RHS has the density of
      measurements from $z^*$ with an exponential scaling. From the above inequality, we can
      replace the expectation over $y|A,x^*$ in $(**)$ with $u| A, z^*$ with an exponential factor.

      Similarly, since posterior sampling now uses $u$ in place of $y$,
      we can replace $\xhat$ in $(**)$ with $\wh{z}$.
  \end{enumerate}

  Combining Case 1 and 2 gives
  \begin{align*}
    (**) & \leq e^{-\Omega(m)} + e^{\left( \frac{4\varepsilon\left( \varepsilon+2\sigma \right)m}{2\sigma^2} \right)} \E_{x^*,z^*} \E_{A} \left[ \E_{u| A, z^*}\left[ \E_{\widehat{z} \sim P(\cdot| u, A)} [ {1}_{B_2} ] \right] \right],\\
    & = e^{-\Omega(m)} + e^{\left( \frac{4\varepsilon\left(
    \varepsilon+2\sigma \right)m}{2\sigma^2} \right)} \E_{z^* \sim P} \E_{A} \left[ \E_{u| A, z^*}\left[ \E_{\widehat{z} \sim P(\cdot| u, A)} [ {1}_{B_2} ] \right] \right].
  \end{align*}

  From the above inequality and eqn.~\eqref{eqn: rp winf eqn3}, we have 
  \begin{align*}
    \Pr_{x^*\sim R , \xi , A, \widehat{x} \sim P( \cdot | A, y)}
    \left[ \norm{ x^* - \widehat{x}} \geq d+ \varepsilon \right] &
    \leq  e^{-\Omega(m)}  + e^{\left( \frac{4\varepsilon\left(
    \varepsilon+2\sigma \right)m}{2\sigma^2} \right)} \Pr_{z^* \sim P, \xi,A, \widehat{z} \sim P( \cdot | u, A)}\left[ \norm{z^* - \widehat{z}} \geq d \right].
\end{align*}
\end{proof}

\subsection{Proof of Theorem~\ref{thm: main}}\label{sec: proof main}

\restate{thm: main}
\begin{proof}
	 We know from Lemma~\ref{lemma: wk implies winf} that there exist $R', P', R'', P''$  
	 and a finite distribution $Q$ supported on the set $S$ such that 
	\begin{enumerate}
		\item $\cW_{\infty}(R', P') \leq \frac{\eps}{\delta^{1/p}}$,
		\item $\min \{ \cW_{\infty}(P', Q),\cW_{\infty}(R', Q) \}  \leq \eta$,
		\item $R = (1-2\delta) R' + 2\delta R''$ and $P = (1-2\delta) P' + 2\delta P''$,
		\item $|S| \leq e^{k}.$
	\end{enumerate}

	Suppose $\cW_{\infty}(P', Q) \leq \eta$. If not, then $\cW_{\infty}(R', Q) \leq \eta$, and by (1),  we see that $\cW_{\infty}(P', Q) \leq \eta + \frac{\eps}{\delta^{1/p}}$, and we will use this in the proof instead. This gives us 
	\begin{align}\label{eqn: rp thm eqn 1}
	\Pr_{x^*\sim R , \widehat{x} \sim P( \cdot | y)} \left[ \norm{x^* - \widehat{x}} \geq (c+1) \eta + (c+1) \sigma \right] &
	\leq \Pr_{x^*\sim R , \widehat{x} \sim P( \cdot | y)} \left[
	\norm{x^* - \widehat{x}} \geq (c+1) \eta + c \sigma +
	(\eps/\delta^{1/p}) \right] \nonumber\\
    &  \leq 2\delta + (1-2\delta) \Pr_{x^*\sim R' ,\widehat{x} \sim P(
		\cdot | y)} \left[ \norm{x^* - \widehat{x}} \geq (c+1) \eta + c
		\sigma + (\eps/\delta^{1/p}) \right] ,
  \end{align}
	where the first line follows since $\sigma \geq \varepsilon /
	\delta^{\frac{1}{p}}$, and the second line follows by decomposing $R
	= (1- 2\delta) R' + 2 \delta R''.$

	We now bound the second term on the right hand side of the above equation.
	For this term, consider the joint distribution over $x^*, A , \xi, \widehat{x}$.
	By Lemma~\ref{lemma: rp winf},  we can replace $x^* \sim R'$ with $z^* \sim P'$, 
	replace $y = Ax^* + \xi$ with $u = A z^* + \xi,$ and 
	replace $\xhat \sim P( \cdot | A,y)$ with $\wh{z} \sim P( \cdot | A,u)$
	to get the following bound
  \begin{align}
    & \Pr_{x^*\sim R' , A , \xi, \widehat{x} \sim P( \cdot | A, y)} \left[ \norm{x^* - \widehat{x}} \geq \left( c + 1 \right)\eta + c\sigma + (\eps/\delta^{1/p}) \right] \leq \nonumber \\
		& e^{-\Omega(m)} + e^{\left( \frac{2(\eps/\delta^{1/p})\left( (\eps/\delta^{1/p}) + 2\sigma \right)m}{\sigma^2} \right)}\Pr_{z^*\sim P', A , \xi, \widehat{z} \sim P(\cdot | u, A)}\left[ \norm{ z^* - \widehat{z}} \geq (c+1)\eta + c \sigma \right]\label{eqn: rp thm eqn 2}.
  \end{align}

    We now bound the second term in the right hand side of the above inequality.
    Let $\Gamma$ denote an optimal $\cW_\infty-$coupling between $P'$ and $Q$. 

    For each $\tilde{z} \in S$, the conditional coupling can be defined as
    \[
    \Gamma( \cdot | \tilde{z}) = \frac{\Gamma( \cdot, \tilde{z})}{Q(\tilde{z})}.
    \]
    By the $\cW_\infty$ condition, each $\Gamma( \cdot | \tilde{z})$ is supported on a ball of radius $\eta$ around $\tilde{z}$.

    Let $ E = \{ z^*, \widehat{z} \in \R^n: \norm{ z^* - \widehat{z}} \geq \left( c+1 \right)\eta+ c\sigma\}$ denote the event that $z^*, \widehat{z}$ are far apart.
  By the coupling, we can express $P'$ as 
  \begin{align*}
    P' &= \sum_{\tilde{z} \in S} Q(\tilde{z}) \Gamma( \cdot | \tilde{z}).
  \end{align*}
  This gives 
  \begin{align*} 
    \Pr_{z^* \sim P', A , \xi, \widehat{z} \sim P(\cdot | A, u)} \left[ E \right] & = \sum_{\tilde{z}^* \in S} Q ( \tilde{z}^*) \E_{z^*\sim \Gamma( \cdot | \tilde{z}^*) , A , \xi, \widehat{z} \sim P(\cdot | A, u) } \left[ {1}_E \right].
  \end{align*}
  
    For each $\tilde{z}^* \in S$, we now bound $Q ( \tilde{z}^*) \E_{z^*\sim \Gamma( \cdot | \tilde{z}^*) , A , \xi, \widehat{z} \sim P(\cdot | A, u) } \left[ {1}_E \right].$ 

		For each $\tilde{z}^* \in S$, we can write $P$ as $P = \left( 1-2\delta \right) Q_{\tilde{z}^*} P_{\tilde{z}^*, 0} + c_{\tilde{z}^*, 1 } P_{\tilde{z}^*, 1} + c_{\tilde{z}^*,2} P_{\tilde{z}^*,2}$, where the components of the mixture are defined in the following way. The first component $P_{\tilde{z}^*,0}$ is $\Gamma( \cdot | \tilde{z}^*)$, the second component is supported within a $c(\eta + \sigma)$ radius of $\tilde{z}^*$, and the third component is supported outside a $c\left( \eta+\sigma \right)$ radius of $\tilde{z}^*$.

		Formally, let $B_{\tilde{z}^*}$ denote the ball of radius $c(\eta
		+ \sigma)$ centered at $\tilde{z}^*$, and let $B^c_{\tilde{z}^*}$
		be its complement. The constants are defined via the following
		Lebesque integrals, and the mixture components for any Borel
		measurable $B$ are defined as
		\begin{align*}
		  c_{\tilde{z}^*, 1} &:= \int_{B_{\tilde{z}^*}} dP - \left( 1-2\delta \right) Q_{\tilde{z}^*} \int_{B_{\tilde{z}^*}} d\Gamma( \cdot | \tilde{z}^*)  ,\\
		  \nonumber\\
		  c_{\tilde{z}^*, 2} &:= \int_{B_{\tilde{z}^*}^c} dP - \left( 1-2\delta \right) Q_{\tilde{z}^*} \int_{B_{\tilde{z}^*}^c} d\Gamma( \cdot | \tilde{z}^*)  ,\\
		  \nonumber\\
		  P_{\tilde{z}^*, 0}(B) &:= \Gamma( B \cap B_{\tilde{z}^*} | \tilde{z}^*) = \Gamma( B  | \tilde{z}^*) \text{ since } \supp(\Gamma(\cdot | \tilde{z}^*)) \subset B_{\tilde{z}^*}, \\
		  \nonumber\\
		  P_{\tilde{z}^*, 1}(B) &:= \begin{cases}
		    \frac{ 1}{c_{\tilde{z}^*, 1}}  P( B \cap B_{\tilde{z}^*})  - \frac{ 1 - 2\delta }{c_{\tilde{z}^*, 1}} Q_{\tilde{z}^*} \Gamma(  B \cap B_{\tilde{z}^*} | \tilde{z}^*)  & \text{ if } c_{\tilde{z}^*, 1} > 0,\\
		    \text{do not care } & \text{ otherwise.}
		  \end{cases},\\
		  \nonumber\\
		  P_{\tilde{z}^*, 2}(B) &:= \begin{cases}
		    \frac{ 1}{c_{\tilde{z}^*, 2}}  P( B \cap B^c_{\tilde{z}^*})  - \frac{ 1 - 2\delta }{c_{\tilde{z}^*, 2}} Q_{\tilde{z}^*} \Gamma(  B \cap B^c_{\tilde{z}^*} | \tilde{z}^*)  & \text{ if } c_{\tilde{z}^*, 2} > 0,\\ \text{do not care } & \text{ otherwise.}
		  \end{cases}.
		\end{align*}
		Notice that if $z^*$ is sampled from $\Gamma(\cdot | \tilde{z}^*)$, then by the $W_\infty$ condition, we have $\norm{ z^* - \tilde{z}^*} \leq \eta$. Furthermore, if $\widehat{z}$ is $\left( c+1 \right)\eta + c\sigma$ far from $z^*$, an application of the triangle inequality implies that it must be distributed according to $P_{\tilde{z}^*,2}$.
    That is, 
    \begin{align*}
			Q ( \tilde{z}^*) \E_{z^*\sim \Gamma( \cdot | \tilde{z}^*) , A , \xi, \widehat{z} \sim P(\cdot | A,u) } \left[ {1}_E \right] & \leq \E_{A , \xi, z^*} \Pr \left[ z^* \sim P_{\tilde{z}^*,0}, \widehat{z} \sim P_{\tilde{z}^*,2}( \cdot | u) \right] \\
			& \leq \frac{1}{1-2\delta}\E_A \left[ 1 - TV( H_{\tilde{z}^*,0}, H_{\tilde{z}^*,2}) \right],
    \end{align*}
    where $H_{\tilde{z}^*,0}, H_{\tilde{z}^*,2}$ are the push-forwards of $P_{\tilde{z}^*,0}, P_{\tilde{z}^*,2}$ for $A$ fixed and the last inequality follows from Claim~\ref{claim: rp w2 thm tv claim}.

		Notice that if we sum over all $\tilde{z}^* \in S $, then the LHS
		of the above inequality is an expectation over $z^* \sim P'$. This
		gives:
    \begin{align*} 
      \Pr_{z^* \sim P', A , \xi, \widehat{z} \sim P(\cdot | u,A)} \left[ E \right] & \leq \frac{1}{1-2\delta} \sum_{\tilde{z}^* \in S} \E_{A} \left[ 1 - TV(H_{\tilde{z}^*, 0}, H_{\tilde{z}^*,2}) \right].
    \end{align*}

    Notice that $P_{\tilde{z}^*,0}$ is supported within an $\eta-$ball around $\tilde{z}^*$, and $P_{\tilde{z}^*,2}$ is supported outside a $c(\eta+\sigma)-$ball of $\tilde{z}^*$. By Lemma~\ref{lemma: noisy pout tv improved} we have 
    \begin{align*}
			\E_A [TV(H_{\tilde{z}^*,0}, H_{\tilde{z}^*,2})] \geq & 1 - 4 e^{-\frac{m}{2}\log\left( \frac{c}{4e^2} \right)}.
    \end{align*}

    This implies 
    \begin{align*}
      \Pr_{z^*\sim P', A, \xi, \widehat{z} \sim P( \cdot | u, A )}\left[\norm{ z^* - \wh{z} } \geq (c+1)\eta+ c\sigma \right] & \leq  \frac{1}{1-2\delta} \sum_{\tilde{z}^*\in S} \E_A \left[ ( 1 - TV(H_{\tilde{z}^*, 0}, H_{\tilde{z}^*,2})) \right], \\
      & \leq \frac{1}{1-2\delta}4 |S|  e^{-\frac{m}{2}\log\left( \frac{c}{4e^2} \right)},\\
			& \leq \frac{1}{1-2\delta}4 e^{-\frac{m}{4}\log\left( \frac{c}{4e^2} \right)},
    \end{align*}
		where the last inequality is satisfied if $m \geq 4\log\left(|S|\right).$

		Substituting in Eqn~\eqref{eqn: rp thm eqn 2}, if $c > 4 \exp{\left( 2+  \frac{8(\eps/\delta^{1/p})\left( (\eps/\delta^{1/p}) + 2\sigma \right)}{\sigma^2} \right)},$ we have
		\begin{align*}
		  \Pr_{x^*\sim R' , A , \xi, \widehat{x} \sim P( \cdot | A, y)} \left[ \norm{x^* - \widehat{x}} \geq \left( c + 1 \right)\eta + c\sigma + (\eps/\delta^{1/p}) \right] \leq & e^{-\Omega(m)} + \frac{1}{1-2\delta} e^{-\Omega(m \log c)} .
		\end{align*}

		This implies that there exists a set $S_{A,\xi}$ over $A,\xi$ satisfying $\Pr_{A,\xi}[S_{A,\xi}] \geq 1 - e^{-\Omega(m)},$ such that for all $A , \xi \in S_{A, \xi},$ we have
		\begin{align*}
		  \Pr_{x^*\sim R', \widehat{x} \sim P( \cdot | y)}\left[\norm{ x^* - \wh{x} } \geq (c+1)\eta+ c\sigma + (\eps/\delta^{1/p}) \right] & \leq  \frac{1}{1 - 2\delta}e^{-\Omega(m)}.
		\end{align*}

		Substituting in Eqn~\eqref{eqn: rp thm eqn 1}, we have
		\begin{align*}
			\Pr_{x^*\sim R, \widehat{x} \sim P( \cdot | y)}\left[\norm{ x^* - \wh{x} } \geq (c+1)\eta+ c\sigma + (\eps/\delta^{1/p}) \right] & \leq 2\delta +  \frac{1}{1-2\delta}e^{-\Omega(m)} \leq 2\delta + 2 e^{-\Omega(m)}.
		\end{align*}
		Rescaling $c$ gives us our result.
		
		At the beginning of the proof, we had assumed that $\cW_\infty(P',Q)\leq \eta$.
		If instead $\cW_\infty(R',Q) \leq \eta$, then we need to replace $\eta$ in the above bound by $\eta + \frac{\eps}{\delta^{1/p}}$.
		Rescaling $c$ in the above bound gives us the Theorem statement.

\end{proof}

\begin{claim}\label{claim: rp w2 thm tv claim}
  Consider the setting of the previous theorem. We have 
    \begin{align}
      \E_{A , \xi, z^*} \Pr \left[ z^* \sim P_{\ztilde^*, 0}, \wh{z} \sim P_{\ztilde^*,2}( \cdot | u) \right] & \leq \frac{1}{1-\delta_2}\E_A \left[ 1 - TV( H_{\ztilde^*,0}, H_{\ztilde^*,2}) \right],
    \end{align}
\end{claim}

\begin{proof}
  For a fixed $A$, let $h_0, h_2$ denote the corresponding densities of the push forward of $P_{\ztilde^*,0}, P_{\ztilde^*,2}$.
  Then we have
    \begin{align}
      \E_{A , \xi, z^*} \Pr \left[ z^* \sim P_{\ztilde^*,0}, \zhat \sim P_{\ztilde^*,2}( \cdot | u) \right] &= \E_A \int \frac{Q_{\ztilde^*} h_{\ztilde^*,0}(u) c_{\ztilde^*,2} h_{\ztilde^*,2}(u)}{ \left( 1-\delta_2 \right) Q_{\ztilde^*,0} h_{\ztilde^*,0}(u) + c_{\ztilde^*,1} h_{\ztilde^*,1}(u) + c_{\ztilde^*,2} h_{\ztilde^*,2}(u)} du,\\
      &\leq \E_A \int \frac{Q_{\ztilde^*} h_{\ztilde^*,0}(u) c_{\ztilde^*,2} h_{\ztilde^*,2}(u)}{ \left( 1-\delta_2 \right) Q_{\ztilde^*,0} h_{\ztilde^*,0}(u) + c_{\ztilde^*,2} h_{\ztilde^*,2}(u)}  du,\\
      &\leq \E_A \int \frac{Q_{\ztilde^*} h_{\ztilde^*,0}(u) c_{\ztilde^*,2} h_{\ztilde^*,2}(u)}{ \left( 1-\delta_2 \right) Q_{\ztilde^*,0} h_{\ztilde^*,0}(u) + (1-\delta_2) c_{\ztilde^*,2} h_{\ztilde^*,2}(u)} du,\\
      &\leq \E_A \frac{1}{1-\delta_2}\int \frac{Q_{\ztilde^*} h_{\ztilde^*,0}(u) c_{\ztilde^*,2} h_{\ztilde^*,2}(u)}{  Q_{\ztilde^*,0} h_{\ztilde^*,0}(u) +  c_{\ztilde^*,2} h_{\ztilde^*,2}(u)} du,\\
      &\leq \E_A \frac{1}{1-\delta_2}\int \frac{Q_{\ztilde^*} h_{\ztilde^*,0}(u) c_{\ztilde^*,2} h_{\ztilde^*,2}(u)}{ \max \{ Q_{\ztilde^*,0} h_{\ztilde^*,0}(u) \, ,\,  c_{\ztilde^*,2} h_{\ztilde^*,2}(u) \} } du,\\
      &= \E_A \frac{1}{1-\delta_2}\int \min \{Q_{\ztilde^*} h_{\ztilde^*,0}(u) , c_{\ztilde^*,2} h_{\ztilde^*,2}(u) \} du, \\
      &\leq \E_A \frac{1}{1-\delta_2}\int \min \{ h_{\ztilde^*,0}(u) ,  h_{\ztilde^*,2}(u) \} du, \\
      &= \frac{1}{1-\delta_2}\E_A \left[ 1 - TV( H_{\ztilde^*,0}, H_{\ztilde^*,2}) \right].
    \end{align}
  
\end{proof}


\section{Lower Bound Proofs}\label{app:lower}

\subsection{Proof of Lemma~\ref{lemma: MI}}\label{sec: proof MI}
\restate{lemma: MI}
\begin{proof}
  First, we consider the case where $A$ is a deterministic matrix.

  We have $y = Ax^* + \xi.$ Let $z = A x^*,$ which gives $y = z + \xi.$ 

  We have $z_i = a_i^T x^*$ where $a_i$ is the $i^{th}$ row of $A$,
  and $y_i = z_i + \xi_i$.  Since $x^*$ is supported within the sphere
  of radius $r$, we have $\E[z_i^2] = \E[\inner{a_i , x}^2] \leq
  \norm{a_i}^2 r^2$.  Since the Gaussian noise $\xi$ has
  variance $\sigma^2/m$ in each coordinate, every coordinate of $y_i$
  is a Gaussian channel with power constaint $\norm{a_i}^2 r^2$ and
  noise variance $\sigma^2/m$.  Using Shannon's AWGN
  theorem~\cite{cover2012elements,polyanskiy2014lecture,shannon1948mathematical},
  the mutual information between $y_i, z_i,$ is bounded by
 \begin{align*}
   I(y_i ; z_i) \leq \frac{1}{2} \log \left( 1 + \frac{\norm{a_i}^2
   r^2 m}{\sigma^2} \right).
 \end{align*}

 The chain rule of entropy and sub-addditivity of entropy implies,
 \begin{align*}
   I(y ; z) &= h(y) - h(y | z) = h(y) - h( y - z | z ),  \\
   &= h(y) - h( \xi| z ) = h(y) - \sum h( \xi_i| z,\xi_1, \cdots, \xi_{i-1}  ),  \\
   &= h(y) - \sum h( \xi_i), \\
   &\leq \sum h(y_i) - \sum h( \xi_i), \\
   &= \sum h(y_i) - \sum h( y_i | z_i), \\
   &= \sum I(y_i; z_i),  \\
   &\leq \sum_{i=1}^m \frac{1}{2} \log\left( 1 + \frac{\norm{a_i}^2 r^2 m}{\sigma^2}  \right), \\
   &\leq \frac{m}{2}\log\left( 1 + \frac{ m r^2 \norm{A}_\infty^2 }{\sigma^2} \right).
 \end{align*}

 Since $x^* \to z \to y$ is a Markov chain, we can conclude that
 \begin{align*}
   I(y; x^*) \leq I(y; z) \leq \frac{m}{2} \log \left( 1 + \frac{m r^2
   \norm{A}_\infty^2 }{\sigma^2}  \right).
 \end{align*}

 Now, if $A$ is a Gaussian matrix with i.i.d. entries drawn from
 $\cN(0,1/m)$, then the power constraint is $\E[\inner{a_i, x}^2] \leq
 r^2 / m$. This gives us
 \begin{align}
   I(y; z) \leq \frac{m}{2} \log \left( 1 + \frac{r^2}{\sigma^2}
   \right).\label{eqn: MI 1}
 \end{align}
 Now since $A$ is a random matrix, we cannot directly apply the Data
 Processing Inequality of $x^*, y, z$ as before, and 
 need to prove that $I(x^*; y | A) \leq I( y; z)$.
 
 Consider the mutual information $I(x^*, A , z; y)$. By the chain rule of mutual information, we have
 \begin{align*}
   I(x^*, A , z; y) &= I(A; y) + I(x^*; y | A) + I(z; y | x^*, A),\\
   &= I(A; y) + I(z; y | A) + I(x^*; y | z, A),\\
   \Leftrightarrow I(x^*; y | A) + I(z; y| x^*, A) &= I(z; y | A) + I(x^*; y | z, A).
 \end{align*}

 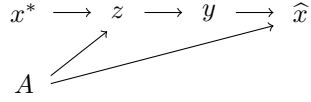
\begin{figure}[t]
   \begin{center}
     \begin{tikzpicture}[
       roundnode/.style={circle, minimum size=7mm},
       ]
       \node[roundnode] (1) {$x^*$};
       \node[roundnode, below=2mm of 1] (2) {$A$} ;
       \node[roundnode, right=5mm of 1] (3) {$z$} ;
       \node[roundnode, right=5mm of 3] (4) {$y$} ;
       \node[roundnode, right=5mm of 4] (5) {$\xhat$} ;
       \draw[->] (1.east) -- (3.west);
       \draw[->] ([yshift=3pt]2.east) -- ([xshift=-4pt, yshift=3pt]3.south);
       \draw[->] (3.east) -- (4.west);
       \draw[->] (4.east) -- (5.west);
       \draw[->] (2.east) -- ([yshift=-5pt]5.west);
     \end{tikzpicture}
   \end{center}
   \caption{DAG relating $x^*, A , z, y, \xhat$. The conditional independencies we use are $x^* \ci y | z, A$ and $A \ci y | z$.}
   \label{fig:CI}
 \end{figure}

 From Figure~\ref{fig:CI}, note that $x^*, y,$ are conditionally independent given $z, A.$
 This gives $I(x^*; y | z, A) = 0$.

 This gives
 \begin{align}
   I(x^*; y | A) + I(z; y| x^*, A) &= I(z; y | A),\\
   \Rightarrow I(x^* ; y | A) &\leq I(z; y | A).\label{eqn: MI 2}
 \end{align}

 We can bound $I(z; y | A)$ in the following way.
 \begin{align}
   I(A,z ; y ) &= I(A ; y) + I( z ; y | A),\\
   &= I(z ; y) + I( A; y | z),\\
   \Leftrightarrow I( A; y) + I(z; y | A) &= I(z; y) + I( A ; y | z),\\
   \Leftrightarrow I( A; y) + I(z ; y | A) &= I(z;y),\\
   \Rightarrow I(z ; y | A) & \leq I(z ; y),\label{eqn: MI 3}.
 \end{align}
 where the second last line follows from $I( A ; y | z)=0$, and the last line follows from $I(A;y)\geq 0$.

 From Eqn~\eqref{eqn: MI 1},~\eqref{eqn: MI 2},~\eqref{eqn: MI 3}, we have
 \begin{align*}
   I(x^*; y | A) &\leq \frac{m}{2}\left( 1 + \frac{m r^2
   \norm{A}_\infty^2}{\sigma^2} \right).
 \end{align*}

\end{proof}

\subsection{Proof of Lemma~\ref{lemma: DPI}}\label{sec: proof DPI}
\restate{lemma: DPI}

\begin{proof}
  When $A$ is a deterministic matrix, 
  the proof follows directly from the Data Processing
  Inequality~\cite{cover2012elements}. Since $x^* \to y \to \xhat$ is
  a Markov chain, we get
  $$ I( x^*; \xhat ) \leq I( y; x^* ).$$

  Now when $A$ is a random matrix, we need to show $I(x^*; \xhat) \leq
  I(y; x^* | A)$.
  Consider the mutual information $I(x^*; y, A , \xhat)$.
  By the chain rule of mutual information, we can express it in two ways:
  \begin{align}
    I(x^*; y, A , \xhat) &= I(x^*; y, A) + I(x^*; \xhat | y, A),\label{eqn: DPI 1}\\
    &= I(x^*; \xhat) + I(x^*; y, A | \xhat).\label{eqn: DPI 2}
  \end{align}

  As $\xhat$ is a function of $y,A,$  we have $I(x^*; \xhat | y,A) = 0$.
  Also, $I(x^*; y,A | \xhat ) \geq 0$.
  Substituting in Eqn~\eqref{eqn: DPI 1},~\eqref{eqn: DPI 2}, we have
  \begin{align*}
    I(x^*; \xhat ) &\leq I(x^*; y, A), \\
    &= I(x^*; A) + I(x^*; y | A), \\
    &= I(x^*; y | A),
  \end{align*}
  where the second line follows from the chain rule of mutual information, 
  and the last line follows because $x^*,A,$ are independent.

\end{proof}
\subsection{Proof of Fano variant Lemma~\ref{lem: covering continuous}}\label{sec: proof covering continuous}

We will build up Lemma~\ref{lem: covering continuous} in sequence.
Before showing it in its full generality, we will show when
$x, \widehat{x},$ are discrete random variables and $x$ is uniform
(Lemma~\ref{lemma: covering uniform}.  We then lift the uniformity
restriction on $x$ (Lemma~\ref{lemma: covering nonuniform}) before
extending to continuous distributions (Lemma~\ref{lem: covering
  continuous}).

\begin{lemma}\label{lemma: covering uniform}
  Let $Q$ be the uniform distribution over an arbitrary discrete finite set $S$.
  Let $(x, \widehat{x})$ be jointly distributed, where $x \sim Q$ and $\widehat{x}$ is distributed over an arbitrary countable set, satisfying
  \begin{align*}
    \Pr \left[ \norm{x - \widehat{x}} \leq \varepsilon  \right] \geq 1 - \delta. 
  \end{align*}
  
  Then for all $\tau \in (0,1)$, we have 
  \begin{align*}
   \tau(1-\delta)\log \cov_{2\varepsilon, \delta + \tau}(Q) &\leq I(x; \xhat ) + 2.  
  \end{align*}

\end{lemma}

\begin{proof}
Let $E = \bm{1}\{ \norm{x - \xhat} \leq \varepsilon \}$ be the
indicator random variable for $x$ and $\xhat$ being close.

Via claim~\ref{claim:uniform1}, we get
\begin{align}
  H(x|E=1) & \geq \log |S| - \frac{1}{1-\delta}.\label{eqn:uniform1}
\end{align}

Recall,
\begin{align*}
  I(x; \xhat | E=1) &= H(x| E=1) - H(x| \xhat, E=1)
\end{align*}

By the Law of total probability, we have:
\begin{align*}
  I(x; \xhat | E=1) &= \sum_{v} \Pr[\xhat = v | E=1] \left(H(x|E=1) - H(x|
  \xhat=v, E=1)\right). 
\end{align*}

We would like to apply a version of Markov's inequality to the above
equation. However, the terms in the summation could be negative.
However, from~\eqref{eqn:uniform1} we have that $H(x|E=1) +
\frac{1}{1-\delta} \geq \log |S|$. Furthermore, since $x$ is
supported on a discrete set of cardinality $|S|$, we have
$H(x|\xhat=v, E=1) \leq \log |S|$. Adding and subtracting
$\frac{1}{1-\delta}$, in the above equation, we have
\begin{align*}
  I(x; \xhat | E=1) &= \sum_{v} \Pr[\xhat = v | E=1] \left(H(x|E=1) +
    \frac{1}{1-\delta} -
  H(x| \xhat=v, E=1) - \frac{1}{1-\delta} \right) ,\\
   &= \sum_{v} \Pr[\xhat = v | E=1] \left(H(x|E=1) + \frac{1}{1-\delta} -
   H(x| \xhat=v, E=1)\right) - \frac{1}{1-\delta} ,\\
    \Leftrightarrow I(x; \xhat | E=1) + \frac{1}{1-\delta}  &=
    \sum_{v} \Pr[\xhat = v | E=1] \left(H(x|E=1) +
    \frac{1}{1-\delta} - H(x| \xhat=v, E=1)\right) 
\end{align*}

 Since the above summation has only non-negative terms that average to
 $I(x; \widehat{x} | E=1) + \frac{1}{1-\delta}$, for all $\tau
 \in (0,1)$, there exists $G_1 \subseteq \supp(\widehat{x})$ with
 $\Pr[G_1 | E=1] \geq 1-\tau,$ such that for all $v \in G_1$, we have
 \begin{align*}
   H(x|E=1) + \frac{1}{1-\delta}  - H(x | \widehat{x} = v,
   E=1) &\leq \frac{I(x; \widehat{x} | E=1) +
   \frac{1}{1-\delta} }{\tau}.
 \end{align*}

 From~\eqref{eqn:uniform1}, we have $H(x|E=1) + \frac{1}{1- \delta}  \geq \log
 |S|$. Hence for all $v\in G_1$, we have
 \begin{align}
   \log |S| - H(x | \widehat{x} = v, E = 1) &\leq \frac{I(x;\xhat |
   E=1) + \frac{1}{1-\delta} }{\tau} ,\nonumber\\ 
   \Leftrightarrow H(x | \xhat=v, E=1)&\geq \log|S| - \frac{I(x;
   \widehat{x}) + \frac{1}{1-\delta} }{\tau},\nonumber\\ 
 \Rightarrow \log \lvert \supp(x | \widehat{x} = v, E=1) \rvert &\geq \log
 |S| - \left( \frac{I(x;\xhat) + \frac{1}{1-\delta} }{\tau}
 \right),\nonumber\\
 \Rightarrow \log \lvert S \cap B(v,\varepsilon) \rvert &\geq \log
 |S| - \left( \frac{I(x;\xhat) + \frac{1}{1-\delta} }{\tau}
 \right),\label{eqn:uniform2}
 \end{align}
 where the last inequality follows as conditioned on $E=1$, $x$ must
 be supported on an $\varepsilon$-radius ball around $\xhat$.

 Now consider the set $G_2 = (S\times G_1) \land E_1$. That is,
 $G_2\subseteq \supp(x, \xhat)$, such that $(u,v) \in G_2$ if and only
 if $\norm{u - v} \leq \epsilon$ and $u \in S, v \in G_1$. 
 Since $\Pr[E_1] \geq 1-\delta$ by the statement of the lemma, and
 $\Pr[G_1 | E_1] \geq 1 - \tau$ by construction, we have
 \[
   \Pr[G_2] \geq (1-\delta) (1-\tau) \geq 1 - \delta - \tau.
 \]

 Now for all $(u,v) \in G_2$, we have
 \begin{align}
   \norm{u - v} & \leq \varepsilon, \nonumber\\
   \log | S \cap B(v,\varepsilon) | & \geq \log |S| - \left( \frac{I(x;
   \xhat | E = 1) + \frac{1}{1-\delta} }{\tau}  \right).
 \end{align}

 Note that by the construction of $G_2$, the set $\bigcup_{v\in G_2}
 B(v,\varepsilon)$ covers a $1-\delta-\tau$ fraction of $S$. As each
 ball $B(v,\varepsilon)$ also has a large intersection with $S$, 
 by the pigeon-hole principle, any $2\varepsilon$-packing of this
 $1-\delta-\tau$ fraction of $S$ must have size at most $2^{\left(I(x;
 \xhat | E=1) + \frac{1}{1-\delta} \right) / \tau}$. 
 
 Hence, we can find a $2\varepsilon$-cover of a $1-\delta-\tau$
 fraction of $S$ that has size at most $2^{\left(I(x;
 \xhat | E=1) + \frac{1}{1-\delta} \right) / \tau}$. 

 This gives
 \begin{align}
   \log \cov_{2\varepsilon,  \delta + \tau}(Q) &\leq \frac{I(x;
   \xhat | E=1) + \frac{1}{1-\delta}
 }{\tau}.\label{eqn:uniform3}
 \end{align}

 We are almost done, since we now only need to relate $I(x; \xhat | E =1)$
 to $I(x; \xhat)$.

 By the chain rule of mutual information, we have
 \begin{align*}
 I(x; \xhat, E) &= I(x; \xhat) + I(x; E | \xhat) = I(x; E) + I(x;
 \xhat | E),\\
 \Rightarrow I(x; \xhat | E) & \leq I(x; \xhat) + I(x; E | \xhat), \\
                             &\leq I(x; \xhat) + 1,\\
 \Leftrightarrow I(x; \xhat | E=0) \Pr[E=0] + I(x; \xhat | E=1)
 \Pr[E=1] &\leq I(x; \xhat) + 1,\\
 \Rightarrow I(x; \xhat | E=1) &\leq \frac{I(x;\xhat) + 1}{1-\delta}.
 \end{align*}

 Substituting in Eqn~\eqref{eqn:uniform3}, we have
 \begin{align*}
   \log \cov_{2\varepsilon, \delta + \tau}(Q) &\leq \frac{I(x;
 \xhat ) + 2}{\tau(1-\delta)} , \\
   \Rightarrow \tau(1-\delta)\log \cov_{2\varepsilon, \delta + \tau}(Q) &\leq I(x; \xhat ) + 2.
 \end{align*}

\end{proof}

\begin{claim}\label{claim:uniform1}
  Let $x \sim Q$, where $Q$ is the uniform distribution over an
  arbitrary discrete finite set $S$. Let $E$ be a binary random variable such
  that $\Pr[E=1] \geq 1 -\delta$.

  Then we have
  \begin{align*}
  H(x|E=1) &\geq \log |S| - \frac{1}{1-\delta}.
  \end{align*}
  
\end{claim}
\begin{proof}
  Let $p = \Pr[E=1]$. By the definition of conditional entropy, we
  have
  \begin{align*}
  H(x|E) &= (1-p) H(x|E=0) + p H(x|E=1), \nonumber\\
  \Leftrightarrow H(x|E=1) &= \frac{1}{p} ( H(x|E) - (1-p) H(x | E = 0 ) ),\nonumber\\
                           &= \frac{1}{p} ( H(x) - I(x ; E) - (1-p) H(x | E = 0 ) ),\nonumber\\
                           &= \frac{1}{p} ( \log|S| - I(x ; E) - (1-p) H(x | E = 0 ) ),\nonumber\\
                           &\geq \frac{1}{p} ( \log|S| - I(x ; E) - (1-p) \log |S| ),\nonumber\\
                           &= \log |S| - \frac{I(x;E)}{p},\\
                           &\geq \log |S| - \frac{1}{1-\delta},
  \end{align*}
  where the fourth line follows from $H(x) = \log |S|$ since $x$ is
  uniform, the fifth line follows from $H(x|E=0) \leq \log |S|$ since
  $x$ is supported on a discrete set of size $|S|$, and the last line
  follows from $p\geq 1-\delta$ and $I(x;E) \leq H(E) \leq 1.$
\end{proof}

The previous lemma handled the uniform distribution on $x$. Now we show that a similar result applies if $x$'s distribution has quantized probability values. 
\begin{lemma}\label{lemma: covering nonuniform}
  Let $Q$ be a finite discrete distribution over $N \in \bbN$ points such that for each $u$ in its support, $Q(u) = j \alpha$, where $j \in \bbN$ and $\alpha:= \frac{1}{N_2}$ is a discretization level for $ N_2 \in \bbN$ large enough.

  Let $(x, \widehat{x})$ be jointly distributed, where $x \sim Q$ and $\widehat{x}$ is distributed over a countable set, satisfying
  \begin{align*}
    \Pr \left[ \norm{x - \widehat{x}} \leq \varepsilon  \right] \geq 1 - \delta. 
  \end{align*}
  
  Then we have 
  \begin{align*}
    \tau (1-\delta)\logcov_{2\varepsilon,\tau +\delta}( Q ) \leq I(x;
    \widehat{x}) + 2\delta.
  \end{align*}
\end{lemma}

\begin{proof}
  For each $x$ in the support of $Q$, we know that its probability is an integral multiple of $\frac{1}{N_2}$.
  Hence we can define a new random variable $x' = (x, j), x \in \supp(Q), j \in [N_2]$ and a distribution $Q'$ over $x'$ in the following way:
  \begin{align*}
    Q'( (x, j) ) =\begin{cases}
      \alpha & \text{ if } j\alpha \leq Q(x),\\
      0 & \text{ otherwise }.
    \end{cases}
  \end{align*}
  By definition, $Q'$ is a uniform distribution, and its support is a discrete subset of $\R^n \times \bbN$.

  Define the following norm for $x'$. For $x'_1= (x_1, j_1), x'_2=(x_2, j_2)$, define
  \begin{align*}
    \norm{(x_1, j_1) - (x_2, j_2)} := \norm{x_1 - x_2}.
  \end{align*}
  
  In order to apply Lemma~\ref{lemma: covering uniform} on $Q'$, it suffices to show that $I(x; \widehat{x}) = I(x'; \widehat{x})$.
  
  By the chain rule of mutual information, we have
  \begin{align*}
    I( x' ; \widehat{x}) & = I( (x, j) ; \widehat{x} ) \\
    &= I(x ; \widehat{x}) + I( j ; \widehat{x} | x).
  \end{align*}
  Since $\widehat{x}$ is purely a function of $x$, we have $I(j; \widehat{x} | x) = 0$.
  This gives
  \begin{align*}
    I(x' ; \widehat{x}) & = I(x; \widehat{x}).
  \end{align*}

  Similarly construct a version $\widehat{x}' = (\widehat{x}, 0)$ of $\widehat{x}$, whose second coordinate is identically zero.
  Hence for $ x' = (x, j) \sim Q'$, we have
  \begin{align*}
    \norm{x' - \widehat{x}'} &\leq \varepsilon \text{ w.p. } 1 - \delta, \\
    I(x'; \widehat{x}') &= I(x; \widehat{x})
  \end{align*}

  Applying Lemma~\ref{lemma: covering uniform} on $Q'$, we have
  \begin{align*}
    \tau (1-\delta) \logcov_{2\varepsilon,\tau + \delta}(Q') \leq I(x;
    \widehat{x}) + 2.
  \end{align*}

  Since the support of the first coordinate of $Q'$ is the same as the support of $Q$, we have
  \begin{align*}
    \tau (1-\delta)\logcov_{2\varepsilon, \tau +\delta}(Q) \leq I(x;
    \widehat{x}) + 2.
  \end{align*}
 
\end{proof}

We now prove Lemma~\ref{lem: covering continuous}, which allows $(x, \widehat{x})$ to follow an arbitrary distribution.

\restate{lem: covering continuous}
\begin{proof}
	Let $\varepsilon = \eta,$ which is the error in the 	statement of
	the lemma.
  Let $ \gamma > 0$ be a small enough discretization level to be specified later.
  For every $x, \widehat{x} \in \R^n$, let $\bar{x}, \wh{\bar{x}}$ denoted the rounding of $ x , \widehat{x}$ to the nearest multiple of $\gamma$ in each coordinate.

  Let $\bar{R}$ be the discrete distribution induced by this discretization of $ x $. 
  We can create such a distribution by assigning the probability of each cell in the grid to its corresponding coordinate-wise floor.
  This discretization of the support changes the error between $x, \widehat{x}$ in the following way.
  If $\norm{x - \widehat{x}} \leq \varepsilon$ with probability $ 1 - \delta$, an application of the triangle inequality gives 
  \begin{align}\label{eqn: error Rbar}
    \norm{\bar{x} - \wh{\bar{x}}} \leq \varepsilon + 2\gamma\sqrt{n} \text{ with probability} \geq 1-\delta.
  \end{align}

  We also need to take into account the effect discretizing $x, \widehat{x}$ has on their mutual information.
  Note that since $ \bar{x}$ is a function of $x$ alone, and $\wh{\bar{x}}$ is a function of $\widehat{x}$ alone, by the Data Processing Inequality, we have
  \begin{align}\label{eqn: MI R}
    I( \bar{x}; \wh{\bar{x}}) \leq I(x; \widehat{x}).
  \end{align}

  Note that $\bar{R}$ is a distribution on a discrete but infinite set.
  However, for any $\beta \in (0,1]$, we can find a discrete and finite distribution $Q$ such that $ \bar{R} = (1- c_1) Q + c_1 D$, with $c_1 \leq \beta$ and $D$ is some other probability distribution.
  This is feasible because the probabilites of the infinite support of $\bar{R}$ must sum to $1$, and hence we can find a finite subset that sums to atleast $1-\beta$ for any $\beta \in (0,1]$.
  Note that in this process, we only change the marginal of $\bar{x}$ without changing the conditional distribution of $\wh{\bar{x}} | \bar{x}$.
  Let $I(\bar{x} ; \wh{\bar{x}}), I_{Q}(\bar{x} ; \wh{\bar{x}}), I_{D}(\bar{x} ; \wh{\bar{x}})$ denote the mutual information between $\bar{x}, \wh{\bar{x}}$ when the marginal of $\bar{x}$ is $\bar{R}, Q, D,$ respectively.
  From Theorem 2.7.4 in~\cite{cover2012elements}, mutual information is a concave function of the marginal distribution of $\bar{x}$ for a fixed conditional distribution of $\wh{\bar{x}} | \bar{x}$. An application of Eqn~\eqref{eqn: MI R} gives us, 
  \begin{align}
    I(x; \widehat{x}) \geq I(\bar{x} ; \wh{\bar{x}}) &\geq (1 - c_1) I_{Q}(\bar{x}; \wh{\bar{x}}) + c_1 I_{D}(\bar{x}; \wh{\bar{x}}),\\
    & \geq (1-c_1) I_{Q}(\bar{x}; \wh{\bar{x}}) ,\\
    & \geq (1-\beta) I_{Q}(\bar{x}; \wh{\bar{x}}).\label{eqn: MI Rbar}
  \end{align}

  Now since the finite distribution $Q$ has a TV distance of at most $\beta$ to the countable distribution $R$, using Eqn~\eqref{eqn: error Rbar}, we have 
  \begin{align}
    \norm{\bar{x} - \wh{\bar{x}}} \leq \varepsilon + 2\gamma\sqrt{n} \text{ with probability } \geq 1-\beta- \delta \text{ if } \bar{x} \sim Q.\label{eqn: error Q}
  \end{align}

  In order to apply Lemma~\ref{lemma: covering nonuniform} on the distribution $Q$, we need its probability values to be multiples of some discretization level $\alpha$.
  Let $\alpha$ be a small enough quantization level for the probability values. We will specify the value of $\alpha$ later. We can now express the distribution $Q$ as a mixture of two distributions $Q', Q''$.
  The distribution $Q'$ is obtained by flooring the probability values under $Q$ and renormalizing to make them sum to 1. The distribution $Q''$ is the mass not contained in $Q'$, normalized to sum to $1$.
  Since each element in the support of $Q$ loses at most $\alpha$ mass, the total mass in $Q''$ prior to normalization is at most $\alpha N_\beta$, where $N_\beta$ is the cardinaltiy of the support of $Q$.
  This gives
  \begin{align*}
    Q &= ( 1- c_2 ) Q' + c_2 Q'', \; c_2 \leq \alpha N_\beta.
  \end{align*}
  
  From Eqn~\eqref{eqn: error Q}, we have $\norm{\bar{x} - \wh{\bar{x}}} \leq \varepsilon + 2\gamma\sqrt{n}$ with probability $\geq 1-\beta-\delta$ when $\bar{x} \sim Q$.
  Since $Q'$ has a TV distance of at most $ \alpha N_\beta$ to $Q$, if $\bar{x} \sim Q'$, we have 
  \begin{align}
    \norm{\bar{x} - \wh{\bar{x}}} \leq \varepsilon + 2\gamma\sqrt{n} \text{ with probability } \geq 1-\beta- \delta - \alpha N_\beta \text{ if } \bar{x} \sim Q'.\label{eqn: error Q'}
  \end{align}

  Let $I_{Q}(\bar{x} ; \wh{\bar{x}}), I_{Q'}(\bar{x} ; \wh{\bar{x}}), I_{Q''}(\bar{x} ; \wh{\bar{x}})$ denote the mutual information between $\bar{x}, \wh{\bar{x}}$ when the marginal of $\bar{x}$ is $ Q, Q', Q''$ respectively.
  Mutual information is a concave function of the marginal distribution of $\bar{x}$ for a fixed conditional distribution of $\wh{\bar{x}} | \bar{x}$.
  Hence using Eqn~\eqref{eqn: MI Rbar}, we have
  \begin{align}
    \frac{I(x; \widehat{x})}{1 - \beta} \geq I_{Q}(\bar{x} ; \wh{\bar{x}}) &\geq (1 - c_2) I_{Q'}(\bar{x}; \wh{\bar{x}}) + c_2 I_{Q''}(\bar{x}; \wh{\bar{x}}),\\
    & \geq (1-c_2) I_{Q'}(\bar{x}; \wh{\bar{x}}) ,\\
    & \geq (1- \alpha N_\beta) I_{Q'}(\bar{x}; \wh{\bar{x}}).\label{eqn: MI Q}
  \end{align}

  Hence if $\bar{x}\sim Q'$, we have $I(\bar{x}; \wh{\bar{x}}) \leq \frac{I(x ; \widehat{x})}{(1-\alpha N_\beta)(1-\beta)}.$
  Applying Lemma~\ref{lemma: covering nonuniform} on the distribution $Q'$, for any $\tau>0$, we have 
  \begin{align*}
    \tau (1- \beta - \delta - \alpha N_\beta)\logcov_{2\varepsilon + 4\gamma\sqrt{n}, \tau  +
    \beta + \delta + \alpha N_\beta}(Q') \leq \frac{I(x;
  \widehat{x})}{(1-\alpha N_\beta) (1-\beta)} + 2.
  \end{align*}

  Now since $Q'$ has at least $1 - \alpha N_\beta$ of the mass under $Q$ and $Q$ has at least $ 1 - \delta$ of the mass under $\bar{R}$,
  the mass $ \tau + \beta + \delta + \alpha N_\beta$ not covered under $Q'$ can be replaced with $\tau + \beta + 2\delta + 2\alpha N_\beta$ under $\bar{R}$.
  This gives
  \begin{align*}
    \tau (1- \beta - \delta - \alpha N_\beta) \logcov_{2\varepsilon +
    4\gamma\sqrt{n}, \tau + \beta +  2\delta + 2\alpha
  N_\beta}(\bar{R}) \leq \frac{I(x ; \widehat{x})}{(1-\alpha N_\beta)
(1-\beta)} + 2.
  \end{align*}

  Now since we can cover the whole distribution of $R$ by extending each element in the support of $\bar{R}$ by $\gamma$ in each coordinate,
  we can replace the radius $2\varepsilon + 4\gamma\sqrt{n}$ for
  $\bar{R}$ by $2\varepsilon + 6\gamma\sqrt{n}$ for $R$. 
  This gives
  \begin{align*}
    \tau (1- \beta - \delta - \alpha N_\beta)\logcov_{2\varepsilon + 6
    \gamma\sqrt{n}, \tau + \beta +  2\delta + 2\alpha N_\beta}(R) \leq
    \frac{I(x ; \widehat{x})}{(1-\alpha N_\beta) (1-\beta)} + 2.
  \end{align*}

  For $\gamma = \frac{\varepsilon}{6\sqrt{n}}, \beta = \min\left\{
  \frac{\delta}{3}, 1- \sqrt{0.99} \right\}, \alpha N_\beta =
  \min\left\{ \frac{\delta}{3}, 1-\sqrt{0.99} \right\}$, we have
  \begin{align*}
    0.99 \tau (1 - 2\delta)\logcov_{3\varepsilon , \tau + 3\delta }(R)
    \leq I(x; \widehat{x}) + 1.98.
  \end{align*}
\end{proof}

\subsection{Proof of Theorem~\ref{thm: lower}}

\restate{thm: lower}
\begin{proof}
  Throughout the proof, we use the notation $N(R,\delta)$ to denote a
  minimal set of $3\eta$-radius balls that cover at least $1-\delta$
  mass under the distribution $R$.

  Let $B$ be the ball in $N(R,10\delta)$ with smallest marginal
  probability. If we set $S \leftarrow N(R,10\delta) \setminus B$,
  then $S$ contains smaller than $1-10\delta$ of $R$. 

  Let $R = (1-c) R' + c R''$, where the components $R'$ and $R''$ are
  probability distributions restricted to $S$ and its complement $S^c$
  respectively. By the construction of $S$, we have $c > 10 \delta$.
  Note that since $R''$ contributes at least $10\delta$ to $R$, any
  algorithm that succeeds with probability $\geq 1-\delta$ over $R$
  must succeed with probability $\geq 0.9$ over $R''$. 

  Now consider $x \sim R''$. 
  By Lemma~\ref{lemma: DPI} and Lemma~\ref{lemma: MI}, we have
  \begin{align*}
    I(x; \widehat{x}) &\leq I(x; y | A),\\
                        & \leq \frac{m}{2}\log\left( 1 + \frac{r^2}{\sigma^2} \right).
  \end{align*}

  Applying Lemma~\ref{lem: covering
  continuous} on $R''$ with parameters $\tau = \delta=0.1$,
  for the failure probability, we can conclude that 
  \begin{align}
    0.99 \cdot 0.1 \cdot (1-0.2) \log | N(R'', 0.4) | &\leq I(x;
    \xhat) + 1.98 \leq \frac{m}{2} \log\left(1 + \frac{r^2}{\sigma^2}
    \right) + 1.98, \nonumber\\
      \Leftrightarrow m &\geq \frac{0.1584 \; \log|N(R'', 0.4)| -
      3.96}{\log\left( 1 + \frac{r^2}{\sigma^2}
    \right)}.\label{eqn:lower1}
  \end{align}

  We now need to express the covering number of $R''$ in terms of the
  covering number of $R$.

  Note that as $R''$ contains at least $10\delta$ mass under $R$, 
  $N(R'', 0.4)$ contains at least $6\delta$ mass under $R$, .
  Similarly, since $N(R,10\delta)$ contains at least
  $1-10\delta$ mass under $R$, $N(R'',0.4) \cup N(R,10\delta)$ will
  contain at least $4\delta$ mass under $R$.
  Hence, we get
  \begin{align}
    |N(R'',0.4)| + |N(R,10\delta)| & \geq |N(R,4\delta)|
    \Leftrightarrow |N(R'',0.4)| \geq |N(R,4\delta)| -
    |N(R,10\delta)|. \label{eqn:lower2}
  \end{align}

  Now we need to relate $N(R,4\delta)$ with $N(R,10\delta)$. This can
  be accomplished via a simple counting argument. Assume that the
  balls in $N(R,4\delta)$ are ordered in decreasing order of their marginal
  probability, then the last $\frac{10\delta}{1-4\delta}$-fraction
  of balls in $N(R,4\delta)$ must contain at most $10\delta$ mass. This implies
  that the first $\frac{1-10\delta}{1-4\delta}$-fraction of
  $N(R,4\delta)$ must contain at least $1-10\delta$ mass. This gives:
  \begin{align}
    \frac{1-10\delta}{1-4\delta} N(R,4\delta) \geq
    N(R,10\delta)\label{eqn:lower3}.
  \end{align}

  Combining Eqn~\eqref{eqn:lower2},~\eqref{eqn:lower3}, we get
  \begin{align*}
    |N(R'', 0.4)| & \geq |N(R,4\delta)| - \frac{1-10\delta}{1-4\delta}
    |N(R,4\delta)|,\\
                  & = \frac{6\delta}{1-4\delta} |N(R,4\delta)|,\\
                  & \geq 6 \delta|N(R,4\delta)|,\\
    \Leftrightarrow \log |N(R'', 0.4)| & \geq \log |N(R,4\delta)| +
    \log (6\delta).
  \end{align*}

  Substituting in Eqn~\eqref{eqn:lower1}, we get
  \begin{align*}
    m \geq \frac{0.1584 \left(\log |N(R,4\delta)| +
    \log(6\delta)\right) - 3.96}{\log\left(1+\frac{r^2}{\sigma^2}
\right)} .
  \end{align*}

  Since $|N(R,4\delta)|=\cov_{3\eta,4\delta}(R)$ by definition, this
  completes the proof.
\end{proof}
\section{Experimental Setup}\label{sec: appendix experiments}
\subsection{Datasets and Architecture}
For the compressed sensing experiment in Fig~\ref{fig:celeba-l2} and
the inpainting experiment in
Figure~\ref{fig:inpaint-hair} we used the
256$\times$256 GLOW model~\cite{kingma2018glow} from the official
repository.
The test set for Fig~\ref{fig:celeba-l2} consists of the first 10
images used by~\cite{asim2019invertible} in their experiments.

For the compressed sensing experiment in
Fig~\ref{fig:cs-reconstr-ffhq}, \ref{fig:ffhq-l2},
\ref{fig:ffhq-reconstr}, we used the FFHQ NCSNv2
model~\cite{song2020improved} from the official repository.
The test set for Fig~\ref{fig:ffhq-l2} consists of the images
69000-69017 from the FFHQ dataset (this corresponds to the first 18
images in the last batch of FFHQ images).

In Fig~\ref{fig:celeba-l2} and Fig~\ref{fig:ffhq-l2}, the measurements
have noise satisfying $\sqrt{\E\norm{\xi}^2}=16$ and
$\sqrt{\E\norm{\xi}^2}=4$ respectively.

\subsection{Hyperparameter Selection}

\paragraph{CelebA experiments}
For MAP, we used an Adam and Gradient Descent optimizer.
Langevin dynamics only uses Gradient Descent. 
Each algorithm was run with learning rates varying over $\left[ 0.1,
0.01, 0.001, 5 \cdot 10^{-4}, 10^{-4}, 5\cdot 10^{-5}, 10^{-5}, 5\cdot
10^{-6}, 10^{-6}\right]$.
For MAP and Modified-MAP, we also performed 2 random restarts for the initialization $z_0$.

The value of $\gamma$ in Eqn~\eqref{eqn:modified-map-defn} was varied
over $[0, 0.1, 0.01, 0.001]$ for Modified-MAP. MAP uses the
theoretically defined value of $\frac{\sigma^2}{m}$. 

For Langevin dynamics, we vary the value of $\sigma_i$ according to
the schedule proposed by~\cite{song2019generative}. We start with
$\sigma_1=16.0$, and finish with $\sigma_{10}=4.0$, such that
$\sigma_i$ decreases geometrically for $i\in [10].$ For each value of
$i$, we do 200 steps of noisy gradient descent, with the learning rate
schedule proposed by~\cite{song2019generative}.

In order to select the optimal hyperparameters for each $m$, we chose
the hyperparams that give maximum likelihood for Langevin and MAP. For
Modified-MAP, we selected the hyperparameters based on reconstruction
error on a holdout set of 5 images.

\paragraph{FFHQ experiments}
The NCSNv2 model is designed for Langevin dynamics. It can be adapted
to MAP by simply not adding noise at each gradient step. We tune the
initial and final values of $\sigma$ used in~\cite{song2020improved},
along with the initial learning rate.

Unfortunately, it is computationally difficult to obtain the
likelihood associated with each reconstruction, since the NCSNv2 model
only provides $\nabla \log p(x)$. Although one could, in theory, do
numerical integration to find $p(x)$, we selected the optimal
hyperparameters for each $m$ based on reconstruction error on a
holdout set of 5 images.

For the Deep-Decoder, we used the over-parameterized network described
in~\cite{asim2019invertible}, and tuned the learning rate over $[0.4,
0.004, 0.0004]$, and selected the hyper-parameters that optimized the
reconstruction error on a holdout set of 5 images.

\subsection{Computing Infrastructure}
Experiments were run on an NVIDA Quadro P5000.

\end{document}